\documentclass[lettersize,journal]{IEEEtran}
\usepackage{amsmath,amsfonts}
\usepackage{algcompatible}
\usepackage{algorithm}
\usepackage{algpseudocode}%
\usepackage{array}
\usepackage{comment}
\usepackage{textcomp}
\usepackage{stfloats}
\usepackage{url}
\usepackage{verbatim}
\usepackage{graphicx}
\usepackage{xcolor}
\usepackage{cite}
\usepackage{amsmath}
\usepackage{tabularx}
\usepackage[english]{babel}
\usepackage{amsthm}
\usepackage{multirow}
\usepackage{booktabs}
\usepackage{mathtools}
\usepackage{subcaption}
\usepackage{caption}
\usepackage{color,soul}

\usepackage[breaklinks=true]{hyperref}
\newtheorem{lemma}{\textbf{Lemma}}

\def\BibTeX{{\rm B\kern-.05em{\sc i\kern-.025em b}\kern-.08em
    T\kern-.1667em\lower.7ex\hbox{E}\kern-.125emX}}
\begin{document}

\title{Enabling Scalable Topology Inference in Distribution Systems via Constrained Multi-Source Inference
\\
}

\author{Haoran~Li*,~\IEEEmembership{Member,~IEEE,}
Lihao~Mai*,~\IEEEmembership{Student Member,~IEEE,}
Muhao~Guo*,~\IEEEmembership{Student Member,~IEEE,}
Jiaqi~Wu*,~\IEEEmembership{Student Member,~IEEE,}
Yang~Weng,~\IEEEmembership{Senior Member,~IEEE,}
\thanks{Haoran Li, Lihao Mai, Muhao Guo, Jiaqi Wu, and Yang Weng are with the Department of Electrical, Computer and Energy Engineering, Arizona State University, Tempe, AZ, 85281, USA. E-mail: \mbox{\{lhaoran,lmai7,mguo26,jiaqiwu1,yang.weng\}@asu.edu}. 

}

\thanks{*~The first four authors contributed equally to this work.}
\vspace{-10mm}
}

\maketitle
\begin{abstract}
Accurate distribution system topology is essential for outage localization, voltage analytics, and operation of distribution grids, yet maintaining reliable connectivity records remains challenging in practice due to heterogeneous and imperfect utility data. Existing topology identification methods often rely primarily on electrical similarity or spatial records alone, which become unreliable in dense feeders and under inconsistent metadata conditions. This paper formulates distribution topology identification as a constrained inference problem that refines a utility-provided base topology using heterogeneous evidence while enforcing spatial feasibility and physical operational constraints. Instead of reconstructing connectivity from scratch, the proposed framework detects inconsistent assignments, performs localized reconnection within constrained neighborhoods to ensure scalability, and iteratively enforces physical feasibility to produce operationally consistent topology estimates. In addition, a falsification-driven reliability metric evaluates how strongly each inferred connection is supported relative to alternative feasible assignments, enabling utilities to prioritize verification efforts while preserving system-wide observability. The framework is validated using operational data from three feeders comprising more than $8{,}000$ AMI meters in collaboration with a large U.S. utility. Results demonstrate over $95\%$ topology reconstruction accuracy while significantly reducing computational effort compared with global inference approaches. The study further shows that correlation-based methods alone produce ambiguous assignments in dense urban feeders, whereas combining electrical measurements with spatial and operational constraints enables robust and scalable topology recovery under realistic deployment conditions.
\end{abstract}



\vspace{-4mm}

\section{Introduction}
\label{sec:intro}

Accurate knowledge of distribution system topology is foundational to modern grid operations, yet in practice many utilities face persistent challenges in maintaining reliable connectivity records \cite{li2021distribution, ma2023hd}. Our collaboration with a major U.S. utility reveals an operational reality: frequent field activities, construction updates, emergency restorations, and occasional human errors continuously perturb record systems \cite{weng2016distributed}. These inaccuracies degrade outage localization, voltage analytics, load forecasting, and fault diagnosis, ultimately limiting the value extracted from AMI, DERs, and distribution automation technologies \cite{hosseini2020machine, mai2025guaranteed, tajer2021advanced}. Although topology identification has been widely studied, a gap remains between algorithms developed under idealized assumptions and solutions that remain reliable when inference must operate on heterogeneous, imperfect, and utility-scale data.

A central reason for this gap is that many topology identification methods implicitly assume consistent metadata and clean measurements~\cite{guo2025efficient}. Under such assumptions, signal similarity metrics such as voltage correlation or model-based residual tests provide strong inference cues~\cite{guo2023graph}. However, when measurement quality varies, customer–transformer mappings are inconsistent, or spatial records contain systematic errors, these cues become ambiguous~\cite{blakely2020identifying}. Methods effective in curated environments often become brittle in operational feeders. Importantly, this limitation frequently arises not from the core algorithmic idea itself, but from inference formulations that ignore imperfect information, operational constraints, and the combinatorial nature of large-scale connectivity assignment \cite{oikonomou2022core, geth2023data, wang2017efficient}. Algorithms that perform well at laboratory scale may become computationally or operationally infeasible when deployed across thousands of transformers and millions of nodes~\cite{wu2022spatial, guo2025solar}.

A second failure mode stems from pipelines that aggressively sanitize data before inference. Traditional approaches discard measurements or records that fail quality checks. In real utility systems, however, removing data may severely reduce observability, especially in sparsely instrumented regions or rapidly evolving feeders~\cite{zhang2020topology,li2025exarnn}. Moreover, in dense urban feeders, customers connected to different transformers can exhibit nearly identical electrical behavior, making correlation-only grouping intrinsically ambiguous without additional information. These observations suggest that topology inference should not depend on eliminating imperfect inputs, but rather incorporate uncertain information while quantifying inference reliability so results remain actionable at system scale \cite{wang2016phase, hosseini2020machine}.

Motivated by these deployment-driven limitations, this paper formulates distribution topology identification as a \emph{constrained inference problem} under heterogeneous and imperfect data. Instead of relying solely on signal similarity or post hoc filtering, topology recovery is framed as an inference task that simultaneously enforces (i) electrical consistency suggested by measurements, (ii) spatial feasibility implied by asset records, and (iii) physical and operational feasibility required by grid operation. This perspective transforms topology recovery from heuristic data cleaning into structured inference restricted to physically meaningful solutions \cite{frank2016introduction, li2021physical, mortlock2024adaptive}. Crucially, the formulation is designed with scalability constraints in mind, ensuring inference procedures remain computationally viable at utility scale.


The major contribution of this paper is a scalable topology refinement solver that integrates noisy metadata, operational constraints, and multi-source information into a unified inference framework suitable for industrial-scale deployment. The key novelty lies not in introducing new machine learning models, but in identifying algorithmic properties required for scalable topology inference and customizing inference mechanisms accordingly for distribution systems. Specifically, the solver detects inconsistent assignments, performs constrained reconnections within local neighborhoods to control combinatorial growth, and iteratively enforces feasibility constraints so that intermediate solutions remain within physically meaningful topology spaces. Through this domain-driven customization, the inference process remains computationally tractable for utility-scale systems while maintaining practical accuracy. In this framework, data determine which assignments are likely, while constraints enforce which assignments are physically feasible \cite{li2021physical, mortlock2024adaptive}.


As a secondary contribution, the framework introduces \emph{connection reliability estimation} to support deployment and operational decision-making. Instead of treating all inferred connections equally, the solver evaluates how strongly each assignment outperforms alternative feasible connections, thereby exposing uncertainty rather than obscuring it. Utilities can then prioritize field verification only where ambiguity remains, preserving system-wide observability while directing engineering resources efficiently \cite{weng2016distributed, ma2023hd}. This reliability-driven workflow bridges algorithmic inference and operational engineering practice, enabling scalable automation without sacrificing engineering accountability.

The proposed framework is validated using operational data from three feeders in our industrial partner's service territory, comprising more than $8{,}000$ AMI meters under heterogeneous measurement and metadata quality. The proposed approach achieves over 95\% topology reconstruction accuracy compared with baseline methods while significantly reducing computational effort relative to global inference approaches, demonstrating feasibility for practical deployment. Beyond accuracy improvements, results show that correlation-based methods alone produce ambiguous assignments in dense feeders, whereas jointly leveraging electrical signals, spatial feasibility, and operational constraints enables robust topology recovery at unprecedented deployment scale \cite{wang2017efficient, wang2024joint}. Numerical studies further confirm that solver runtime satisfies operational refresh intervals required in practical feeder management.

The remainder of this paper is organized as follows. Section II introduces the utility data ecosystem and study feeders and formulates the constrained topology inference problem. Section III presents the constraint-guided scalable inference solver and analyzes computational complexity. Section IV reports numerical validation results on operational feeders. Section V concludes the paper.

\vspace{-5mm}
\section{System Description, Data Overview, and Constrained Inference Formulation}
\label{sec:system}

We consider the low-voltage portion of a distribution feeder where each node (such as AMI meters) is connected to exactly one secondary transformer.\footnote{This paper focuses on secondary node--transformer connectivity and does not attempt to reconstruct medium-voltage switching topology.}
Accurate node--transformer connectivity underpins outage localization, voltage analytics, phase balancing, and DER integration. In practice, however, utilities maintain connectivity and phase information across multiple enterprise systems (e.g., GIS, AMI, OMS, billing databases), and the recorded topology evolves continuously due to construction, emergency restoration, customer service changes, and asynchronous database updates. Thus, record-based connectivity is typically \emph{informative but imperfect}.


\vspace{-4mm}
\subsection{Utility Context, Data Ecosystem, and Study Feeders}
\label{subsec:context}

Our operational utility partner maintains operational data across multiple enterprise systems, including Geographic Information Systems (GIS), Advanced Metering Infrastructure (AMI), Outage Management Systems (OMS), and customer/billing databases. Connectivity and phase information stored in GIS serve as the primary reference for downstream applications, but continuous field operations and imperfect synchronization introduce inconsistencies between recorded connectivity and physical field connections. Figure~\ref{fig:feeders} illustrates the three feeders from a large U.S. distribution utility for evaluation. These feeders span substantially different spatial densities and loading patterns, enabling assessment across both geographically compact rural feeders and dense urban secondary networks. 

\begin{figure}[htbp]
    \centering
    \begin{subfigure}[b]{0.15\textwidth}
        \centering
        \includegraphics[width=\textwidth]{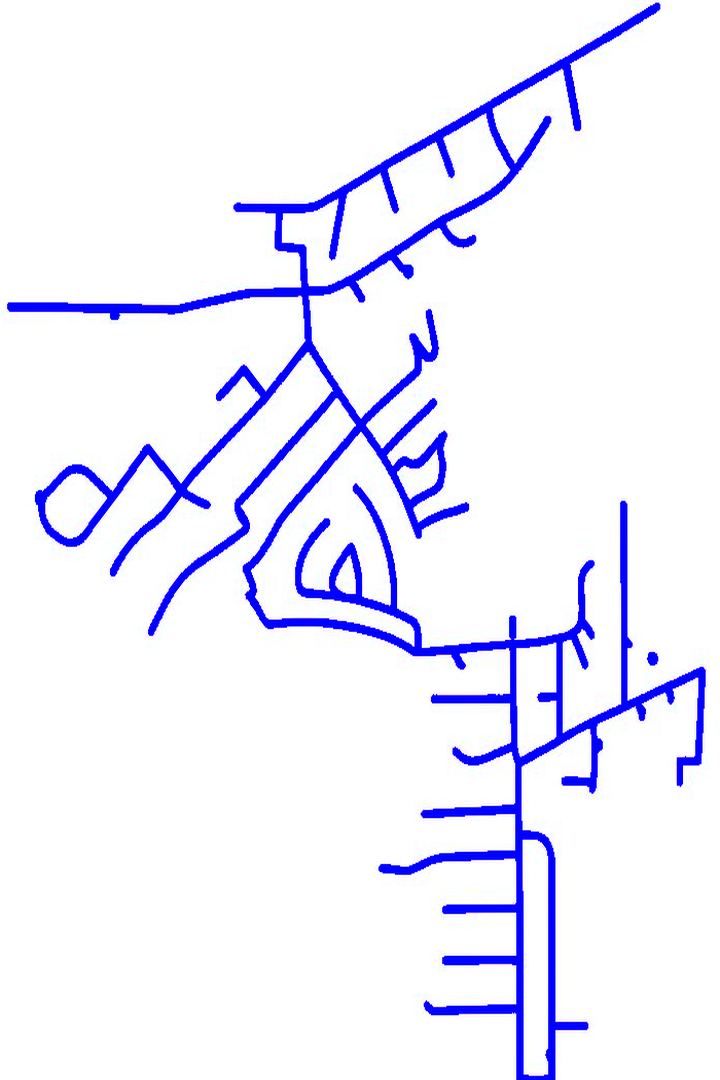}
        \caption{Feeder 1}
        \label{fig:feeder_1}
    \end{subfigure}
    \hfill
    \begin{subfigure}[b]{0.15\textwidth}
        \centering
        \includegraphics[width=\textwidth]{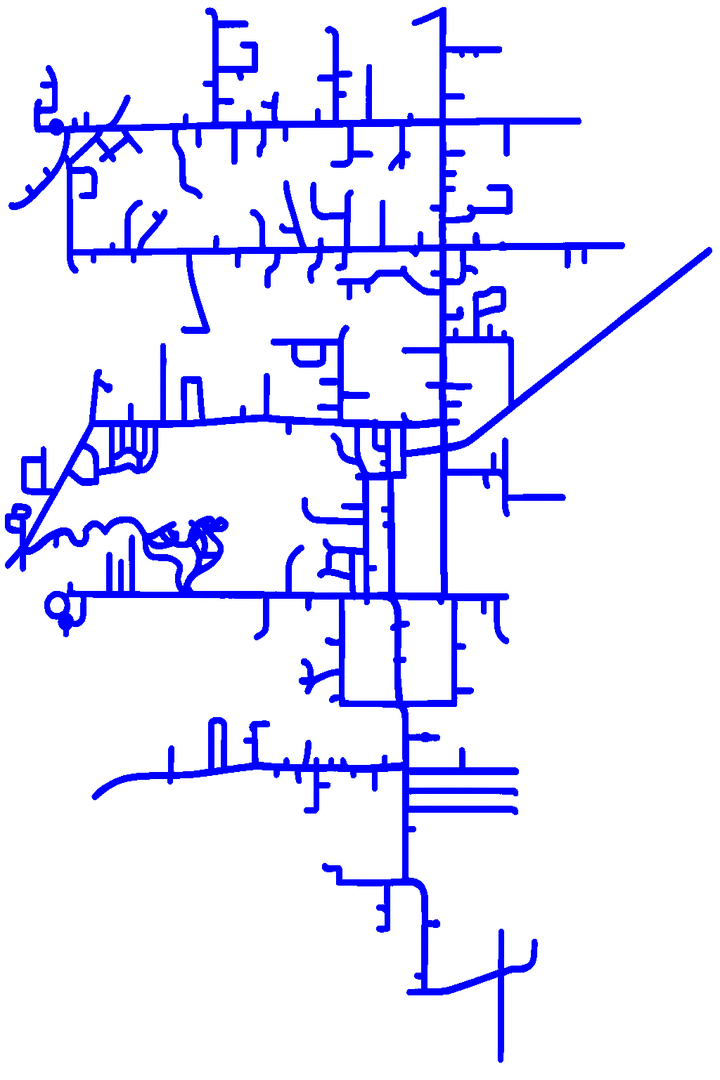}
        \caption{Feeder 2}
        \label{fig:feeder_2}
    \end{subfigure}
    \hfill
    \begin{subfigure}[b]{0.15\textwidth}
        \centering
        \includegraphics[width=\textwidth]{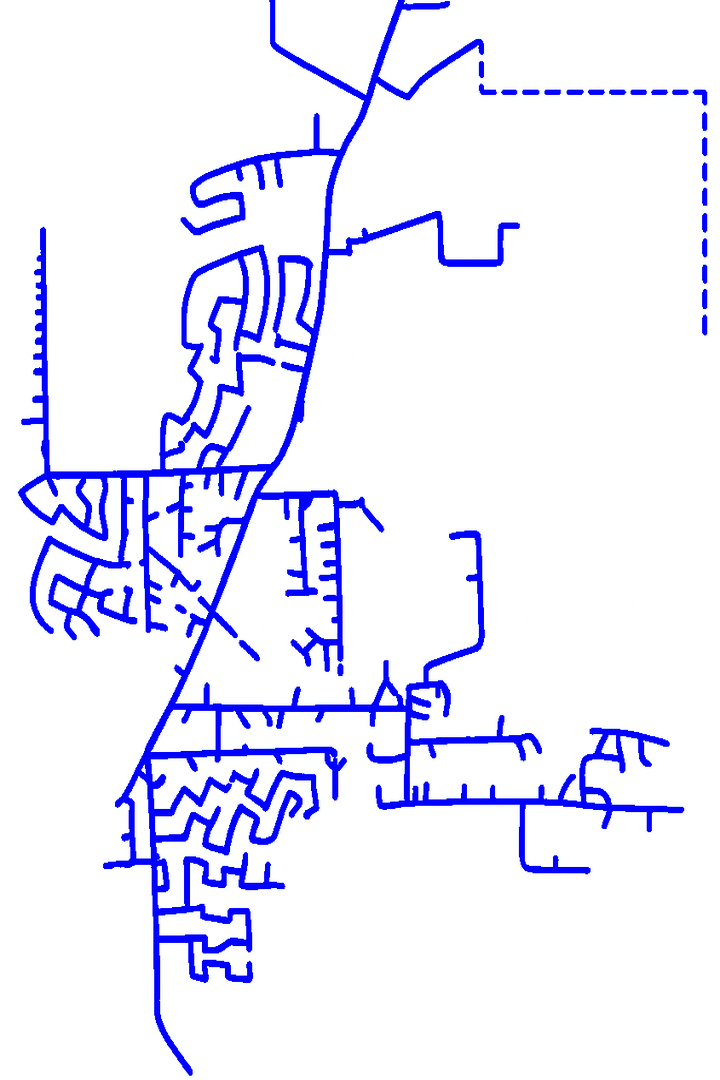}
        \caption{Feeder 3}
        \label{fig:feeder_3}
    \end{subfigure}
    \caption{Feeders in the U.S. provided by our utility partner.}
    \label{fig:feeders}
    \vspace{-3mm}
\end{figure}

\vspace{-1mm}
\subsection{Data Sources and Roles in Inference}
\label{subsec:data_sources}

The inference framework integrates three categories of utility data:
(i) spatial and asset metadata maintained in GIS and record systems,
(ii) AMI voltage and load measurements, and
(iii) outage event logs.
Importantly, none of these sources alone provides fully reliable connectivity information. Instead, each contributes partial evidence whose consistency must be evaluated jointly under operational constraints. Therefore, the objective is not to ``clean'' one dataset in isolation, but to infer connectivity by cross-validating heterogeneous signals while restricting solutions to be physically feasible. Table~\ref{tab:data_sources} summarizes data sources and roles in constrained inference.

\begin{table}[t]
\centering
\caption{Data roles in constrained topology inference.}
\label{tab:data_sources}
\small
\begin{tabular}{p{0.16\columnwidth} p{0.30\columnwidth} p{0.40\columnwidth}}
\toprule
\textbf{Source} & \textbf{Key fields} & \textbf{Role in inference} \\
\midrule
GIS records & transformer IDs, node IDs, base mapping, phases, coordinates, nominal voltages & provides base graph $\mathcal{G}_B$; spatial feasibility and candidate neighborhoods; initial labels with uncertainty \\
AMI & $V_{it}$, $P_{it}$, $Q_{it}$ (15-min) & provides electrical evidence for coupling/consistency; supports local discrimination within candidate sets \\
OMS outage logs & outage/restoration times per node & provides event-driven cross-validation signal for connectivity consistency (optional) \\
\bottomrule
\end{tabular}
\end{table}

\vspace{-3mm}
\subsection{Problem Formulation}
\label{subsec:notation}

Let $\mathcal{V}$ denote the set of nodes (AMI meters) and $\mathcal{T}$ the set of service transformers. The utility provides a base topology and phase record
$\mathcal{G}_{B}=\{\mathcal{V},\mathcal{E}_B,\Phi_B\}$,
where $\mathcal{E}_B$ encodes the base node--transformer mapping and $\Phi_B$ contains phase labels. Because field connectivity evolves over time, $\mathcal{G}_B$ is treated as a noisy prior rather than ground truth. Table~\ref{tab:io_notation} summarizes key inputs and the desired output of the inference process.
\begin{table}[t]
\centering
\caption{Inputs/outputs used in constrained topology inference.}
\label{tab:io_notation}
\small
\begin{tabular}{p{0.22\columnwidth} p{0.70\columnwidth}}
\toprule
\textbf{Symbol} & \textbf{Meaning} \\
\midrule
$(\phi_i,\lambda_i)$ & latitude/longitude for node $i\in\mathcal{V}$ \\
$(\phi_j,\lambda_j)$ & latitude/longitude for transformer $j\in\mathcal{T}$ \\
$\{V^A_{it}\}_t$ & phase-A voltage magnitude time series at node $i$ (similarly for B,C) \\
$\{P_{it}\}_t,\{Q_{it}\}_t$ & active/reactive power time series at node $i$ \\
$\overline{S}_j$ & transformer rating (kVA) for transformer $j$ \\
$\mathcal{C}$ & physical/operational constraint set (capacity, voltage range, phase feasibility, etc.) \\
$\mathcal{G}_B$ & base topology/phase record from utility systems (noisy prior) \\
\midrule
$\hat{\mathcal{G}}=\{\mathcal{V},\hat{\mathcal{E}},\hat{\Phi}\}$ & inferred topology and phase assignment \\
\bottomrule
\end{tabular}
\end{table}
Connectivity is modeled as a bipartite mapping where each node connects to exactly one transformer.
A candidate topology is represented as
$\mathcal{G}=\{\mathcal{V},\mathcal{E},\Phi\}$,
where $\mathcal{E}\subseteq\mathcal{V}\times\mathcal{T}$ is the set of node--transformer edges and $\Phi$ denotes phase assignments.

For each node $i\in\mathcal{V}$, the available heterogeneous observations include:
\textit{Geographical data:} coordinates $(\phi_i,\lambda_i)$;
\textit{Electrical measurements:} voltage time series $\{V^{A}_{it}\}_t$ (and similarly for phases B, C when available);
\textit{Load measurements:} active/reactive power $\{P_{it}\}_t,\{Q_{it}\}_t$;
\textit{Text data:} address string $S_i$;
\textit{Nominal information:} nominal voltage labels when available.
Each transformer $j\in\mathcal{T}$ has rated capacity $\overline{S}_j$.


Let $\mathcal{C}$ denote the operational constraints restricting feasible connectivity. Typical constraints include:
\textit{Voltage feasibility:}
$V^{A}_{it} \in [V_{\mathrm{nom,min}}^{A}, V_{\mathrm{nom,max}}^{A}], \forall i\in\mathcal{V}, \forall t$,
and similarly for phases B and C.
\textit{Transformer capacity feasibility:}
$\sum_{i \in \mathcal{N}(j)} \left(P_{it}+\mathbf{j}Q_{it}\right) \le \overline{S}_{j}, \forall j\in\mathcal{T}, \forall t$,
where $\mathcal{N}(j)=\{i\in\mathcal{V}:(i,j)\in\mathcal{E}\}$.
These constraints define allowable solutions rather than merely serving as post-processing checks.

\textbf{Constrained Inference Objective for Topology Refinement}.
Let $a_i\in\mathcal{T}$ denote the transformer assignment of node $i$, and let $\mathbf{a}=\{a_i\}_{i\in\mathcal{V}}$ collect all assignments (equivalently, $\mathcal{E}$).
Topology refinement is formulated as
\begin{equation}
\label{eq:global_objective}
\min_{\mathbf{a}}
J_{\mathrm{elec}}(\mathbf{a})
+\alpha J_{\mathrm{geo}}(\mathbf{a})
+\beta J_{\mathrm{prior}}(\mathbf{a})
\quad
\text{s.t.}\;
\mathbf{a}\in\Omega(\mathcal{C}),
\end{equation}
where $\Omega(\mathcal{C})$ is the feasible set induced by constraints $\mathcal{C}$ and $\alpha,\beta>0$ trade off spatial plausibility and conservativeness relative to the base record model.

\textbf{Interpretation of the terms (power-language).}
$J_{\mathrm{elec}}$ penalizes assignments that create electrically incoherent transformer groups (e.g., a node whose voltage trajectory is inconsistent with its assigned group under typical operating conditions).
$J_{\mathrm{geo}}$ penalizes geographically implausible assignments given the physical layout of the secondary network.
$J_{\mathrm{prior}}$ penalizes unnecessary deviation from the base record model, reflecting that utilities typically have mostly-correct connectivity and only a small subset of edges are wrong.
The solver outputs a refined topology estimate
$
\hat{\mathcal{G}}=\{\mathcal{V},\hat{\mathcal{E}},\hat{\Phi}\},
$
where $\hat{\mathcal{E}}$ corresponds to an assignment $\hat{\mathbf{a}}$ that (approximately) minimizes \eqref{eq:global_objective} while satisfying feasibility constraints.

\textbf{Localized Optimization for Scalability}. Direct solution of \eqref{eq:global_objective} is combinatorial at utility scale. However, in the deployment setting addressed here, record errors are typically \emph{localized}: most assignments in $\mathcal{G}_B$ remain correct, and inconsistencies concentrate in small regions affected by historical record issues. Therefore, scalable inference proceeds by:
(1) detecting a suspicious node set $\mathcal{O}\subset\mathcal{V}$ whose current assignments contribute disproportionately to $J_{\mathrm{elec}}+\alpha J_{\mathrm{geo}}$,
(2) restricting reassignment candidates for each $i\in\mathcal{O}$ to a local transformer neighborhood (e.g., $K$ nearest transformers),
(3) accepting only updates that reduce the objective while keeping constraints satisfied.
This connects the implemented solver (Section~\ref{sec:solver}) to the constrained optimization: the method is a localized approximation to \eqref{eq:global_objective} that preserves the structure while achieving utility-scale tractability.

\vspace{-3mm}

\section{Utility Data Properties and Failure Modes}
\label{sec:data_challenges}

This section characterizes the practical challenges observed in utility data and explains why purely correlation-based or purely spatial approaches often fail. These observations motivate the constrained inference framework developed later.

\vspace{-3mm}
\subsection{Spatial Record Inconsistencies}
\label{subsec:spatial_issue}

Connectivity records in utilities are typically maintained through GIS and billing systems, where node locations are recorded through manual entry or legacy migration processes. Over time, updates from field activities, customer relocations, or synchronization delays across databases lead to inconsistencies between recorded locations and actual physical connections. Figure~\ref{fig:geocoding} illustrates a representative example from industrial feeders. Several nodes appear to be connected to transformers that are not geographically closest. Such mismatches arise from incorrect or outdated coordinate records, not necessarily from actual wiring errors.

\begin{figure}[h!]
\centering
\includegraphics[width=\columnwidth]{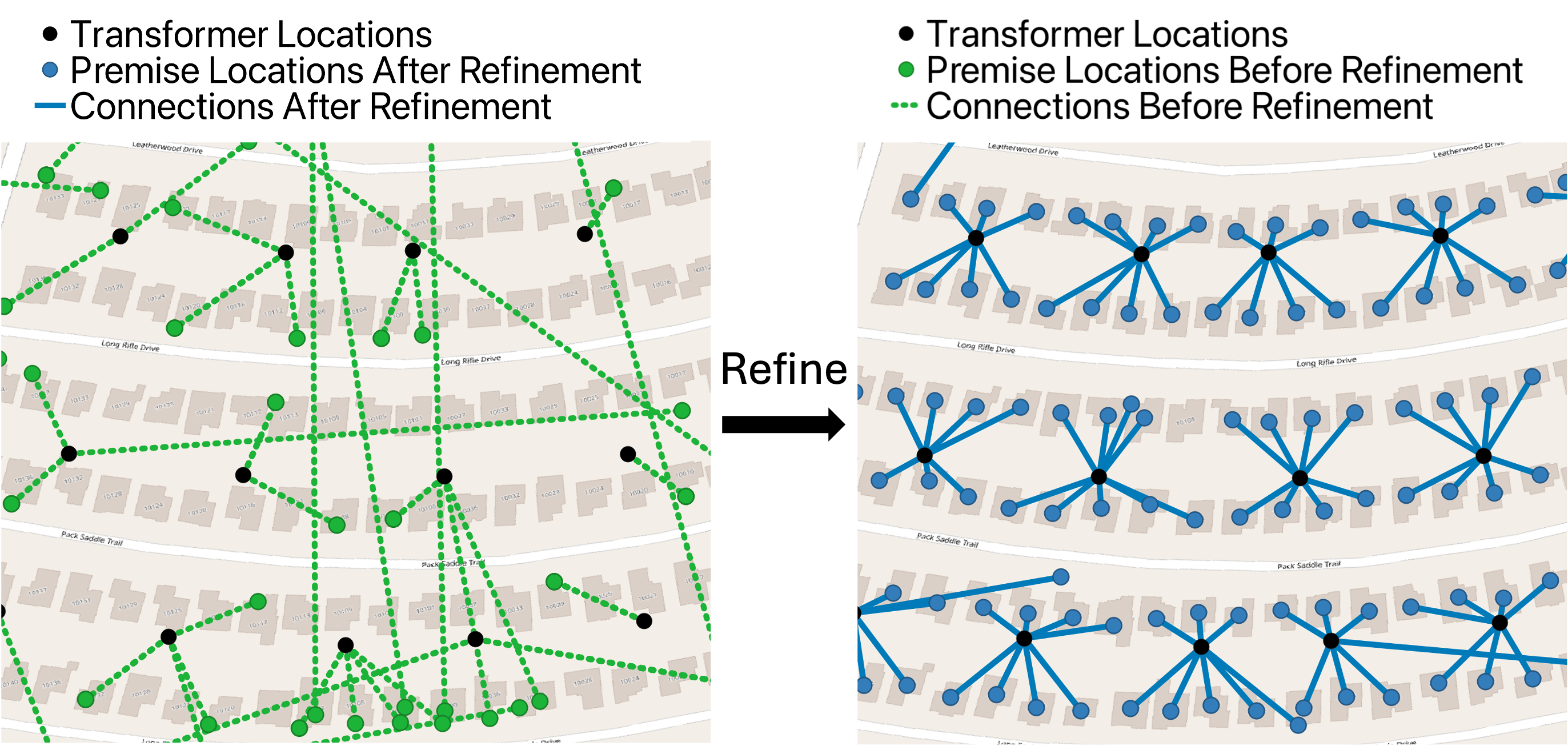}
\caption{Before and after geocoding-based spatial refinement. Nodes incorrectly associated with distant transformers are corrected after address-based coordinate refinement.}
\label{fig:geocoding}
\vspace{-3mm}
\end{figure}

To mitigate these issues, address-based geocoding is used to refine node coordinates and detect cases where spatial assignments contradict physical feasibility. However, spatial proximity alone cannot uniquely determine connectivity, particularly in dense urban feeders where multiple transformers serve geographically overlapping areas. Thus, spatial information is informative but insufficient on its own.

\vspace{-3mm}
\subsection{Measurement Quality and Data Completeness}
\label{subsec:data_quality}

Electrical measurements collected from AMI meters exhibit additional challenges, including missing nominal voltage labels, incomplete time-series data, and physically unrealistic readings caused by sensor faults or communication issues. Figure~\ref{fig:data_filtering} shows voltage profiles before and after cleaning. Some meters report persistently abnormal values, such as unrealistically low voltages or flatlined signals, which degrade correlation analysis if not filtered.

\begin{figure}[h!]
\centering
\includegraphics[width=\columnwidth]{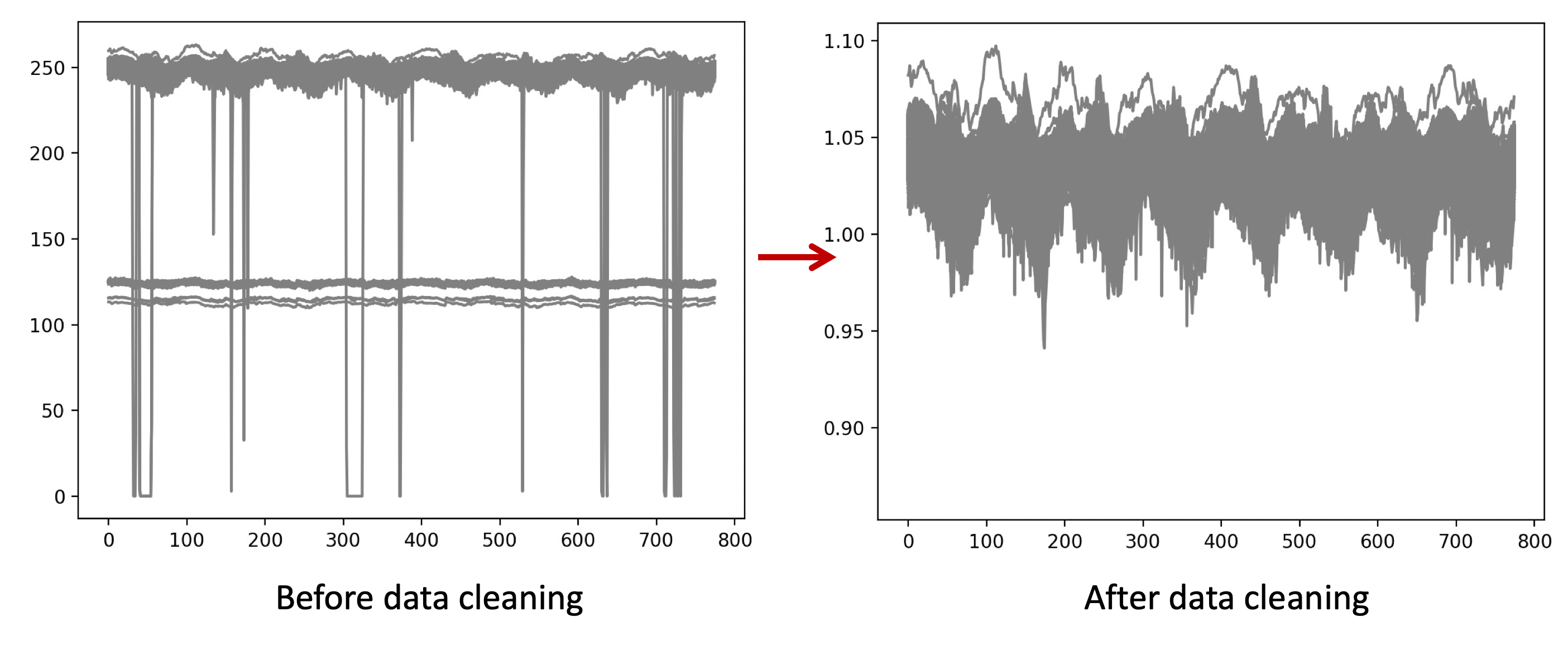}
\caption{Voltage measurements before and after data cleaning. Unrealistic or corrupted measurements must be filtered to preserve electrical consistency analysis.}
\label{fig:data_filtering}
\vspace{-4mm}
\end{figure}

Standard preprocessing steps such as voltage normalization, missing-data filtering, and interquartile-range-based outlier removal are therefore required before electrical similarity metrics can be reliably computed. However, aggressive data removal may reduce system observability. Consequently, preprocessing must balance data cleaning with preserving sufficient coverage for inference.

\vspace{-3mm}
\subsection{Correlation Ambiguity in Dense Feeders}
\label{subsec:correlation_issue}

Many existing topology identification methods rely primarily on voltage correlation among nodes. While correlation is effective in sparsely connected feeders, it becomes unreliable in dense urban environments where multiple nearby transformers experience similar voltage fluctuations. Figure~\ref{fig:flowchart} summarizes the inference pipeline typically used in correlation-based methods. Voltage normalization and correlation computation are followed by clustering and reassignment steps. However, when correlation patterns overlap across nearby transformer groups, clustering alone cannot reliably separate correct assignments. As a result, nodes may be incorrectly grouped despite exhibiting high correlation with multiple candidate transformers. This phenomenon explains why purely signal-based inference often produces ambiguous or unstable assignments in practice.

\begin{figure}[h!]
\centering
\includegraphics[width=\linewidth]{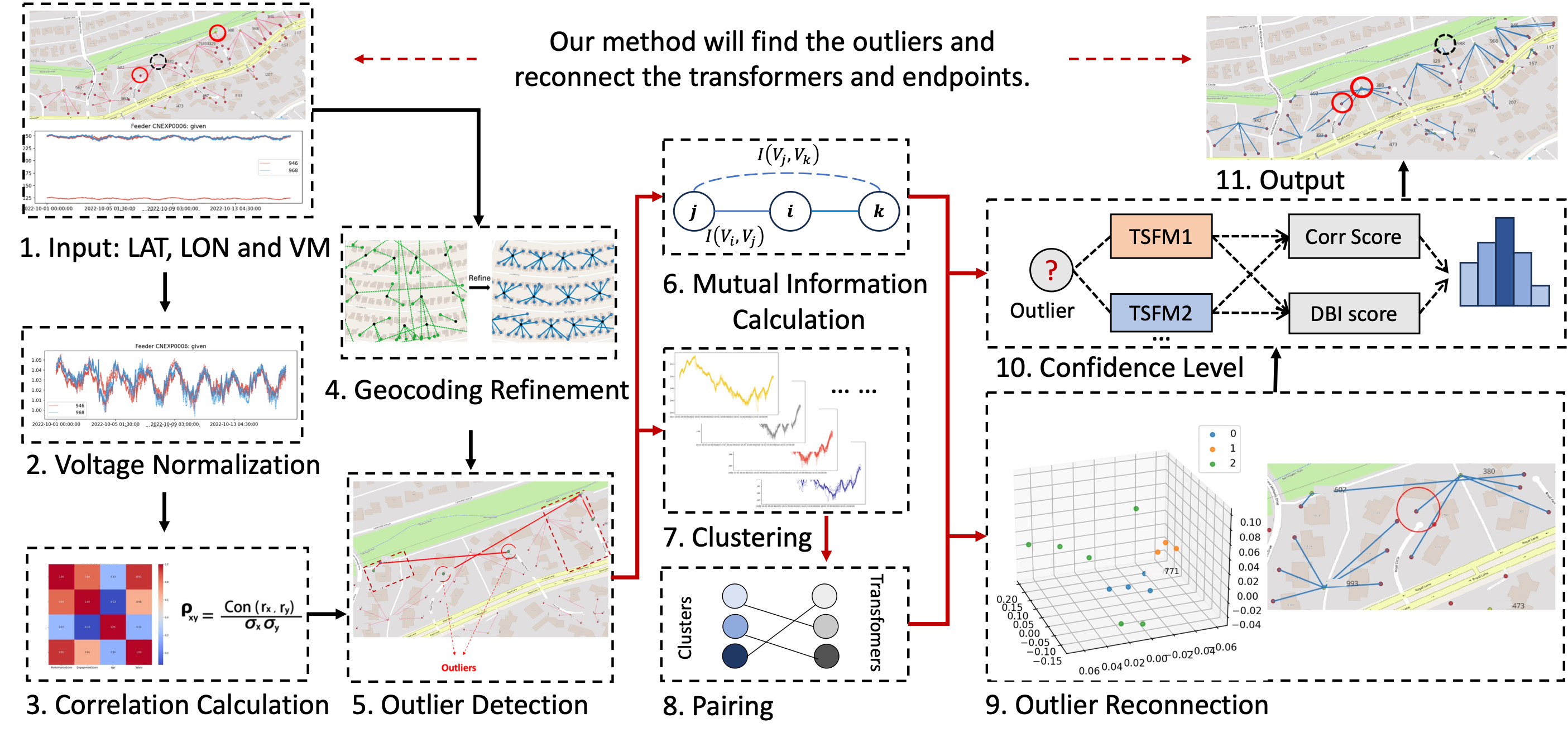}
\caption{Workflow of correlation-based topology identification. Pure correlation analysis struggles in dense feeders.}
\label{fig:flowchart}
\vspace{-3mm}
\end{figure}

\vspace{-3mm}
\subsection{Localized Nature of Topology Errors}
\label{subsec:localized_errors}

Importantly, utility record inaccuracies rarely corrupt entire feeders. Instead, most connectivity remains correct, with errors localized to small subsets of nodes affected by historical record inconsistencies. Figure~\ref{fig:Frameworkorigin} illustrates a practical refinement framework where only suspicious assignments are corrected while preserving reliable portions of the base topology.

\begin{figure}[h!]
\centering
\includegraphics[width=1.08\linewidth]{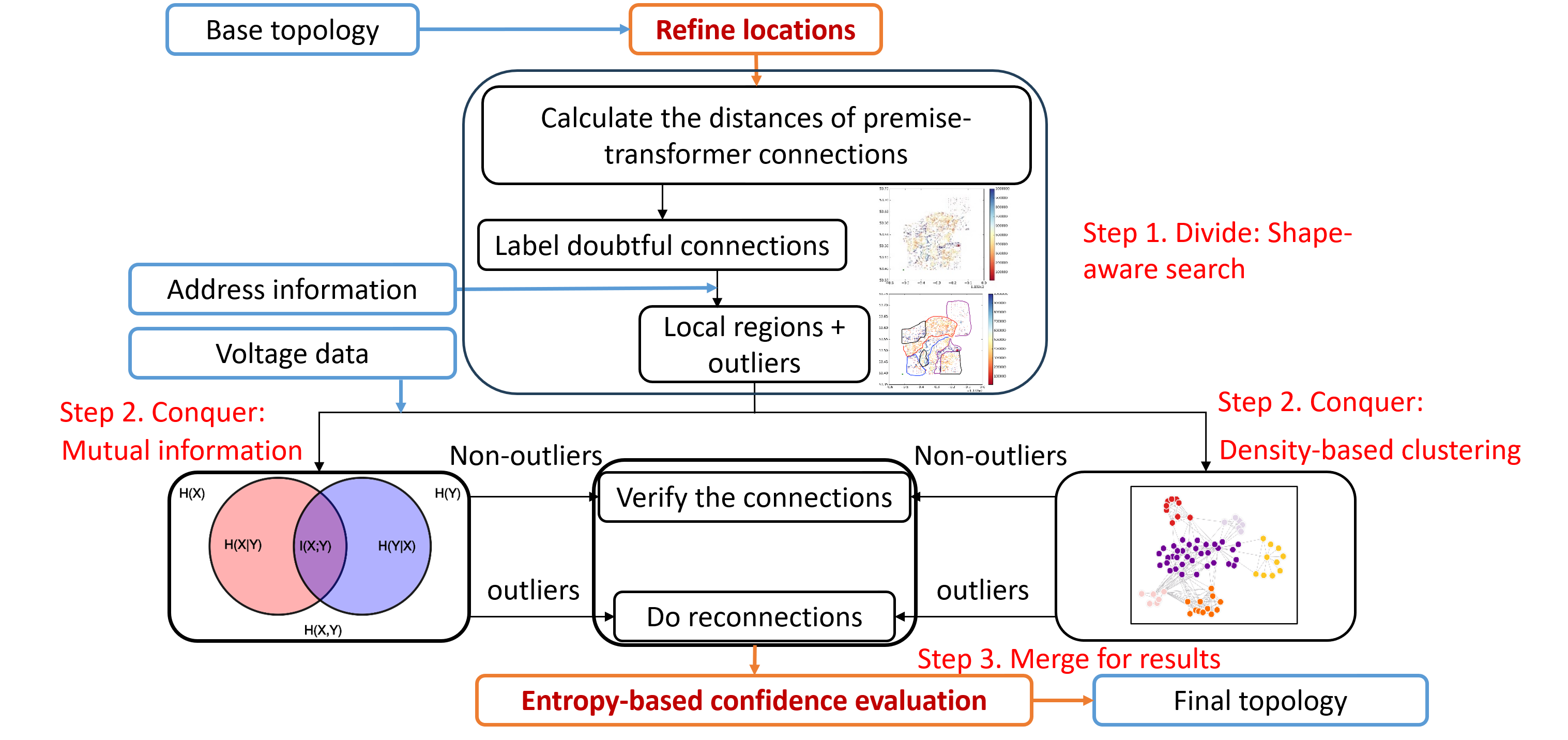}
\caption{Topology refinement framework. Instead of rebuilding topology globally, localized suspicious assignments are detected and corrected within constrained neighborhoods.}
\label{fig:Frameworkorigin}
\vspace{-3mm}
\end{figure}

This observation motivates a divide-and-conquer strategy: detect inconsistent nodes first, then perform localized reassignment rather than global topology reconstruction. Such locality is critical for scalability across feeders containing thousands of nodes.

\vspace{-2mm}
\subsection{Need for Reliability Quantification}
\label{subsec:confidence_need}

Even after refinement, some assignments remain ambiguous due to insufficient distinguishing signals. Utilities therefore require not only an inferred topology but also an estimate of connection reliability to prioritize field verification.

Figure~\ref{fig:confidence} illustrates the concept of evaluating assignments by comparing clustering and correlation quality against falsified alternatives. Assignments that remain superior under falsification tests receive higher confidence. Reliability scoring allows utilities to focus manual verification resources on low-confidence regions while safely deploying high-confidence results.

\begin{figure}[h]
\centering
\includegraphics[width=1.0\linewidth]{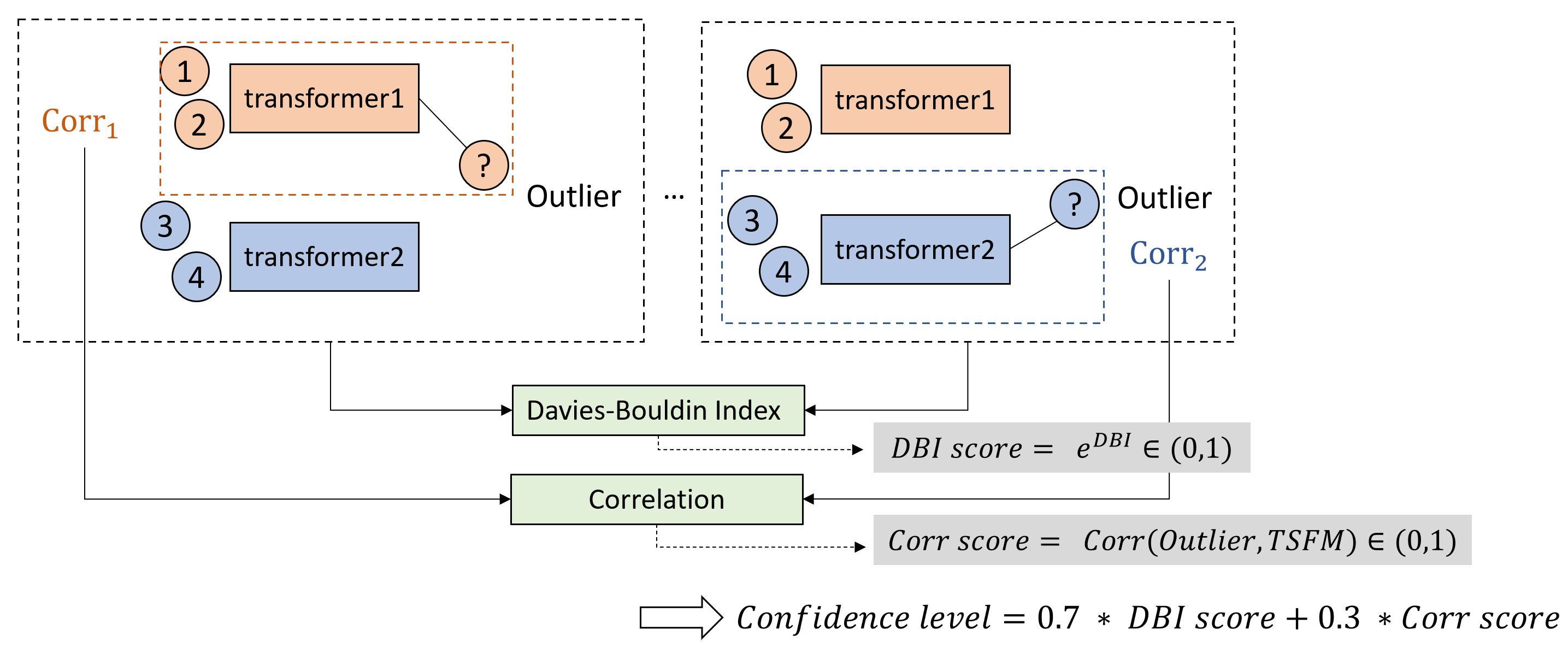}
\caption{Illustration of confidence evaluation via falsification and clustering comparison. Assignments that remain superior to alternatives receive higher confidence.}
\label{fig:confidence}
\vspace{-3mm}
\end{figure}

\vspace{-4mm}
\subsection{Implications for Method Design}

The above observations reveal that practical topology inference must satisfy several requirements simultaneously:

1) Multiple heterogeneous signals must be combined rather than relying on correlation alone.

2) Only localized suspicious assignments should be reconsidered to ensure scalability.

3) Each inferred connection should include a reliability measure to support operational deployment.

However, translating these operational observations into scalable algorithmic design requires identifying the core properties that inference methods must satisfy at utility scale. The following section formalizes these design requirements before presenting the solver implementation.

\section{Design Requirements for Scalable Topology Inference}
\label{sec:design_requirements}

Section~\ref{sec:data_challenges} shows that topology errors arise from inconsistencies among heterogeneous data sources rather than absence of measurements. Consequently, inference algorithms must operate under imperfect data while remaining computationally tractable for feeders containing thousands of nodes. This section formalizes algorithmic requirements necessary for scalable deployment and motivates the solver construction presented next.

\textbf{Requirement 1: Scalable Inference Under Large Systems}. Let $\mathcal{V}$ denote the set of nodes and $\mathcal{T}$ denote the set of transformers. In principle, topology inference seeks assignments
$
f:\mathcal{V}\rightarrow\mathcal{T}
$
mapping each node to a transformer. A global search across all assignments has complexity growing exponentially with $|\mathcal{V}|$, rendering exhaustive inference infeasible for large feeders.

Therefore, scalable inference requires restricting assignment search to candidate subsets
$
f_i:\mathcal{V}_i \rightarrow \mathcal{T}_i,
$
where $|\mathcal{T}_i| \ll |\mathcal{T}|$ for each suspicious node subset $\mathcal{V}_i$. This converts global inference complexity from combinatorial growth to approximately linear scaling in feeder size.

\textbf{Requirement 2: Localized Correction Rather Than Global Reconstruction}. Utility observations indicate that most connections remain correct, with errors localized. Let $\mathcal{E}_B$ denote base topology edges obtained from utility records. Only a small subset $\Delta\mathcal{E}$ contains incorrect assignments:
$
|\Delta\mathcal{E}| \ll |\mathcal{E}_B|.
$ Inference should therefore update only suspicious assignments,
$
\hat{\mathcal{E}} = \mathcal{E}_B \setminus \Delta\mathcal{E}
\cup \Delta\hat{\mathcal{E}},
$
rather than reconstructing $\hat{\mathcal{E}}$ globally. Restricting inference to localized corrections preserves scalability and avoids unnecessary perturbation of reliable connectivity.

\textbf{Requirement 3: Constraint Compatibility During Inference}. Recovered assignments must satisfy operational constraints $\mathcal{C}$, such as transformer capacity and voltage feasibility. Let
$
\mathcal{C}(f)=1
$
indicate feasibility of assignment $f$. Inference must satisfy
$
f^\star = \arg\max_{f} \; \text{Score}(f)
\quad \text{s.t.} \quad \mathcal{C}(f)=1,
$
ensuring physically realizable topology solutions. Constraint compatibility prevents inference from producing statistically plausible but operationally infeasible results.

\textbf{Requirement 4: Reliability Quantification}. Multiple assignments may satisfy physical constraints under limited signal distinguishability. Let $\mathcal{F}_i$ denote feasible assignments for node $i$. Reliability should quantify assignment robustness:
$
R_i = \frac{\text{Score}(f_i^\star)}
{\max_{f \in \mathcal{F}_i \setminus f_i^\star}
\text{Score}(f)},
$
measuring the separation between chosen and competing assignments. Reliability estimation enables utilities to prioritize verification where inference ambiguity remains high.

\textbf{Design Implications}. The above requirements imply that scalable topology inference solvers must: restrict inference to localized candidate sets, preserve reliable base topology portions, enforce physical feasibility constraints, and quantify assignment reliability.


\vspace{-3mm}
\section{Constraint-Guided Topology Inference Solver}
\label{sec:solver}

This section presents the topology inference solver designed explicitly to satisfy those requirements while realizing the previous constrained inference formulation. A key point is that the solver does not introduce new machine learning models; instead, it customizes existing data-driven inference mechanisms so that they become compatible with power-system operational constraints and scalable to utility deployment. 

\vspace{-3mm}
\subsection{Localized Detection of Inconsistent Assignments}

To preserve scalability, inference begins by identifying only those nodes whose assignments contradict spatial or electrical consistency, thereby preserving reliable portions of the base topology $\mathcal{G}_B$. Let node $i$ be assigned to transformer $j$. Using geocoded coordinates, the physical distance is computed:
\begin{equation}
d_{i,j} = \text{GeoDist}((\phi_i,\lambda_i),(\phi_j,\lambda_j)),
\label{eq:geodist_solver}
\end{equation}
and compared with the closest transformer distance
$
d_{i,j^*} = \min_k \text{GeoDist}((\phi_i,\lambda_i),(\phi_k,\lambda_k)).
$
The distance ratio
$
r_i = \frac{d_{i,j}}{d_{i,j^*}}
$
flags spatially implausible assignments when $r_i>\tau$.

Electrical inconsistency is detected using voltage similarity. For node $i$ connected to transformer $j$ with node set $\mathcal{N}(j)$,
$
\rho_{i,k} = \text{Corr}(V_i(t), V_k(t)), \quad k\in\mathcal{N}(j),
$
and assignments exhibiting persistently weak correlation with peers are marked suspicious. Importantly, only inconsistent nodes are reconsidered. This localized correction strategy reflects a domain insight: connectivity errors typically occur in small regions rather than globally. By customizing inference to exploit this property, the algorithm achieves scalability without sacrificing inference quality.

\vspace{-3mm}
\subsection{Localized Candidate Transformer Search}

A major scalability bottleneck in topology inference arises from global search over all candidate transformers. To avoid this, reconnection search is restricted to nearby transformers, leveraging the physical locality of distribution networks. For each suspicious node $i$, candidate transformers are selected as
$
\mathcal{N}_T(i) = \text{NearestTransformers}(i,K),
$
where $K$ geographically closest transformers are considered. Nodes already served by these transformers form candidate dataset
$
\mathcal{D}_i = \bigcup_{k \in \mathcal{N}_T(i)} \mathcal{N}(k).
$

Restricting inference to local neighborhoods ensures computational complexity scales approximately linearly with feeder size, satisfying deployment requirements for large utility systems. This step illustrates how standard clustering tools must be customized to respect feeder geography in order to remain scalable in practice.

\vspace{-3mm}
\subsection{Data-Driven Reconnection Under Physical Constraints}
\label{subsec:reconnection_solver}

To achieve compatibility between data-driven inference and grid operation constraints, reconnection combines clustering-based grouping with feasibility enforcement. Candidate nodes are embedded into a joint spatial–electrical feature space
$
\mathbf{x}_j = [V^A_j(0),\ldots,V^A_j(t), \phi_j, \lambda_j].
$ K-means clustering groups electrically and spatially consistent nodes. Each cluster is then assigned to the transformer most frequently serving its members:
$
T^\star(\mathcal{S}_\ell) =
\arg\max_{k \in \mathcal{N}_T(i)}
|\mathcal{S}_\ell \cap \mathcal{N}(k)|.
$

The outlier node is reconnected according to the transformer serving its cluster. Alternatively, dependency can be evaluated through mutual information:
\begin{equation}
\bar{\text{MI}}(i,k)
= \frac{1}{|\mathcal{N}(k)|}
\sum_{k' \in \mathcal{N}(k)}
\text{MI}(V_i(t), V_{k'}(t)),
\end{equation}
selecting the transformer maximizing statistical dependency. Crucially, assignments are accepted only if resulting configurations satisfy operational feasibility constraints introduced earlier. This coupling of machine learning similarity measures with physical feasibility is a key customization enabling deployment in power systems.

\vspace{-3mm}
\subsection{Reliability Estimation via Falsification Testing}
\label{subsec:confidence_solver}

Beyond inference accuracy, deployment requires knowing which assignments remain uncertain. Reliability is therefore evaluated via falsification testing. Alternative candidate transformers are tested, and clustering quality is compared via Davies–Bouldin Index:
$
\text{Score}_{\text{DBI}}
= \sigma\!\left(\log
\frac{\text{DBI}_{\text{false}}}
{\text{DBI}_{\text{true}}}\right),
$
where $\sigma(\cdot)$ is a sigmoid mapping.
Correlation-based support is similarly evaluated:
$
\text{Score}_{\text{corr}}
= \sigma\!\left(
\frac{\text{Corr}_{\text{ours}}}
{\text{Corr}_{\text{alt}}}
\right).
$
The final confidence score becomes
$
\text{CL} = 0.7\,\text{Score}_{\text{DBI}}
+ 0.3\,\text{Score}_{\text{corr}}.
$ This step converts inference output into actionable operational information, allowing utilities to prioritize verification only where ambiguity remains. Reliability estimation thus forms a critical bridge between automated inference and engineering deployment.

\vspace{-3mm}
\subsection{Physics-Based Validation and Iterative Refinement}
To ensure operational feasibility, inferred assignments must satisfy transformer loading and voltage constraints.
For transformer $j$,
$
S^{agg}_{jt} =
\sum_{i\in\mathcal{N}(j)}
\sqrt{P_{it}^2 + Q_{it}^2}
$
must satisfy
$
S^{agg}_{jt} \le \overline{S}_j.
$
Voltage compatibility is similarly enforced:
$
V_{it} \in
[V_{\text{nom,min}}, V_{\text{nom,max}}].
$
Assignments violating constraints trigger localized reassignment until feasibility is restored. Because correction remains local, scalability is preserved even during iterative refinement.

\textbf{Overall Solver Workflow}. The workflow summarized in Fig.~\ref{fig:Frameworkorigin} shows how inconsistent assignments are detected, corrected locally, evaluated for reliability, and validated through physical constraints to produce final topology estimates. The resulting solver simultaneously satisfies scalability, locality, feasibility, and reliability requirements. By customizing data-driven inference to respect power-system operational structure, the method becomes deployable across large utility feeders while maintaining physically consistent results.

\vspace{-3mm}
\subsection{Computational Complexity and Scalability}
\label{subsec:complexity}

A central deployment requirement identified in Section~\ref{sec:design_requirements} is that topology inference must scale to feeders containing thousands of nodes without requiring global combinatorial search. Direct reassignment of all nodes would require evaluating all node–transformer mappings, resulting in exponential complexity and making global optimization infeasible for utility-scale systems.

The proposed solver instead performs localized refinement. Let $N$ denote the number of nodes (such as residential loads), $M$ the number of transformers, and let $K$ denote the number of geographically nearby candidate transformers considered for reassignment, where typically $K \ll M$ due to feeder locality. Let $\mathcal{O}$ denote the set of suspicious nodes detected for reassignment. Empirically, topology errors are localized and therefore $|\mathcal{O}| \ll N$, although worst-case scaling remains proportional to $N$. Each reassignment step evaluates electrical similarity and feasibility only within the local candidate set.

\begin{lemma}[Solver Scalability]
Assume localized topology errors such that $|\mathcal{O}| = \gamma N$ for some small constant $\gamma \ll 1$, and that the candidate transformer set size $K$ is bounded by feeder locality. Then the computational complexity of each solver iteration scales as
$
O(|\mathcal{O}| K),
$
which becomes approximately linear in feeder size, i.e.,
$
O(NK).
$
Since $K$ remains small and independent of $N$ in practical feeders, solver complexity scales approximately linearly with system size.
\end{lemma}

\begin{proof}
For each suspicious node $i \in \mathcal{O}$, reassignment evaluates similarity measures and feasibility constraints against at most $K$ candidate transformers and their connected nodes. Therefore, per-node computational effort scales as $O(K)$. Because reassignment is restricted to nodes in $\mathcal{O}$, total computation per iteration scales as $O(|\mathcal{O}|K)$. Under localized error assumptions where $|\mathcal{O}|$ grows proportionally with feeder size, total complexity becomes $O(NK)$. Since $K$ is bounded by physical locality and does not grow with feeder size, complexity becomes approximately linear in $N$.
\end{proof}

This result explains why the solver remains tractable even for feeders with thousands of nodes: inference cost grows primarily with the number of nodes requiring correction rather than with all possible assignments. 

\vspace{-3mm}
\section{Numerical Results}
\label{sec:experiment}

In this section, we first summarize datasets and evaluation setup, then demonstrate (i) failure modes of electrical-only inference, (ii) recovery via constraint-guided inference, (iii) reliability improvements for deployment, (iv) validation through physical feasibility constraints, and (v) computational scalability required for utility-scale operation.

\vspace{-3mm}
\subsection{Implementation and Evaluation Setup}

The datasets used in this study capture spatial, electrical, and temporal characteristics of low-voltage distribution feeders, including node-to-transformer mappings, geolocation information, phase labels, and AMI measurements sampled at 15-minute resolution. Table~\ref{tab:feeder_stats} summarizes feeder scales used in this study, covering both sparse and dense service regions. Figure~\ref{fig:feeders} illustrates feeder layouts.


\begin{table}[H]
\centering
\caption{Feeder Summary Statistics}
\label{tab:feeder_stats}
\begin{tabular}{c|c|c}
\toprule
\textbf{Feeder ID} & \textbf{Number of Transformers} & \textbf{Number of Nodes} \\
\midrule
Feeder 1 & 115  & 1,398 \\
Feeder 2 & 1,548 & 3,677 \\
Feeder 3 & 1,046 & 5,841 \\
\bottomrule
\end{tabular}
\vspace{-3mm}
\end{table}

\vspace{-3mm}
\subsection{Failure of Electrical-Only Inference in Dense Feeders}

Electrical correlation alone often fails in dense urban feeders where neighboring transformers experience nearly identical loading and voltage conditions, resulting in indistinguishable electrical signatures. Figure~\ref{fig:Correlation test case} shows nodes connected to different transformers exhibiting correlations above 0.95, preventing separation by correlation-only methods even though physical connectivity differs. These examples demonstrate that electrical evidence alone is insufficient for large-scale topology inference, motivating constraint-guided inference combining electrical, spatial, and physical feasibility signals.

\begin{figure}
    \centering
    \begin{subfigure}[b]{\columnwidth}
        \centering        \includegraphics[width=0.5\columnwidth]{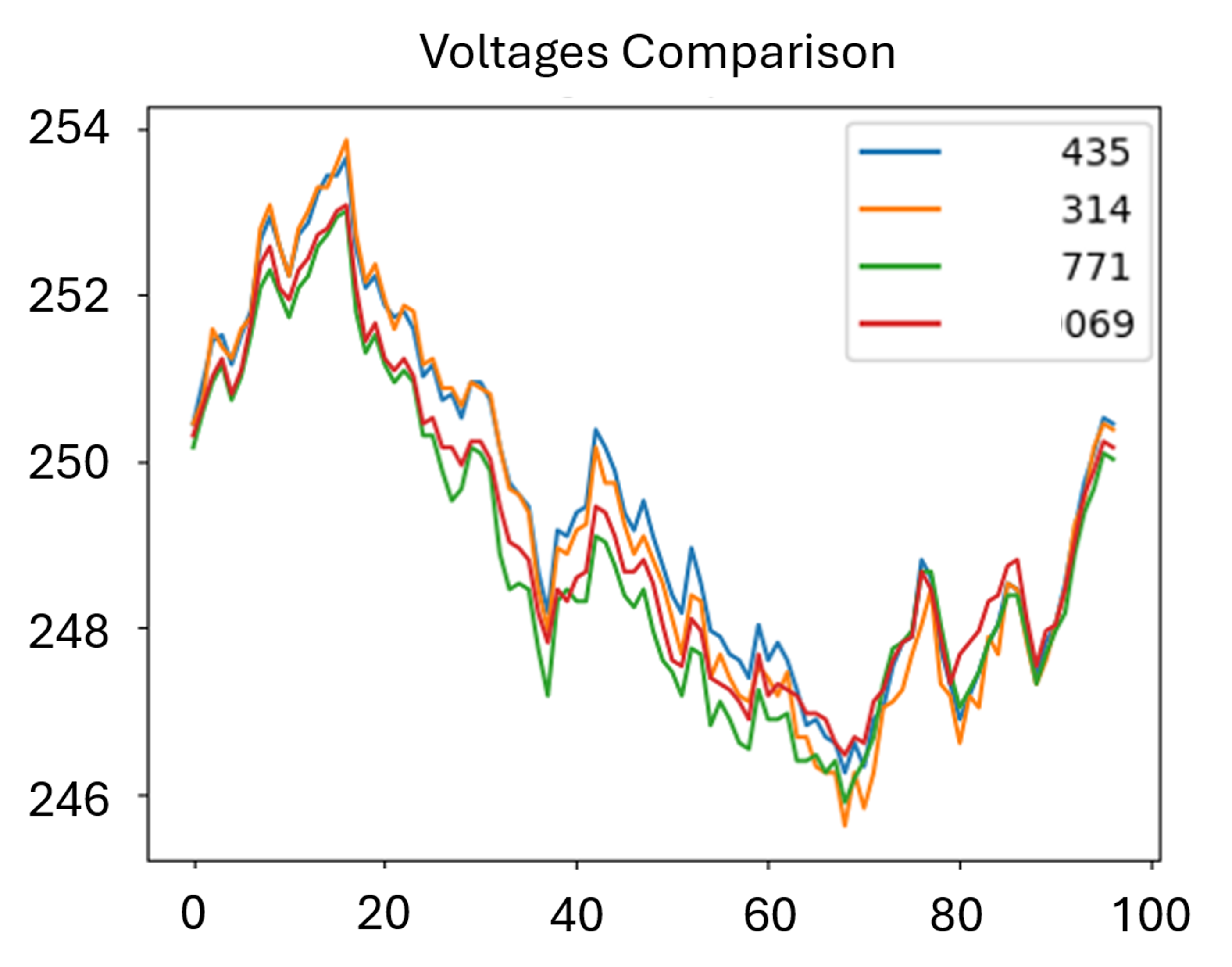}
        \caption{Voltage correlation between two transformers.} \label{fig:correlation377109415and75893988}
    \end{subfigure}
    \hfill
    \begin{subfigure}[b]{0.49\textwidth}
        \centering
        \includegraphics[width=\textwidth]{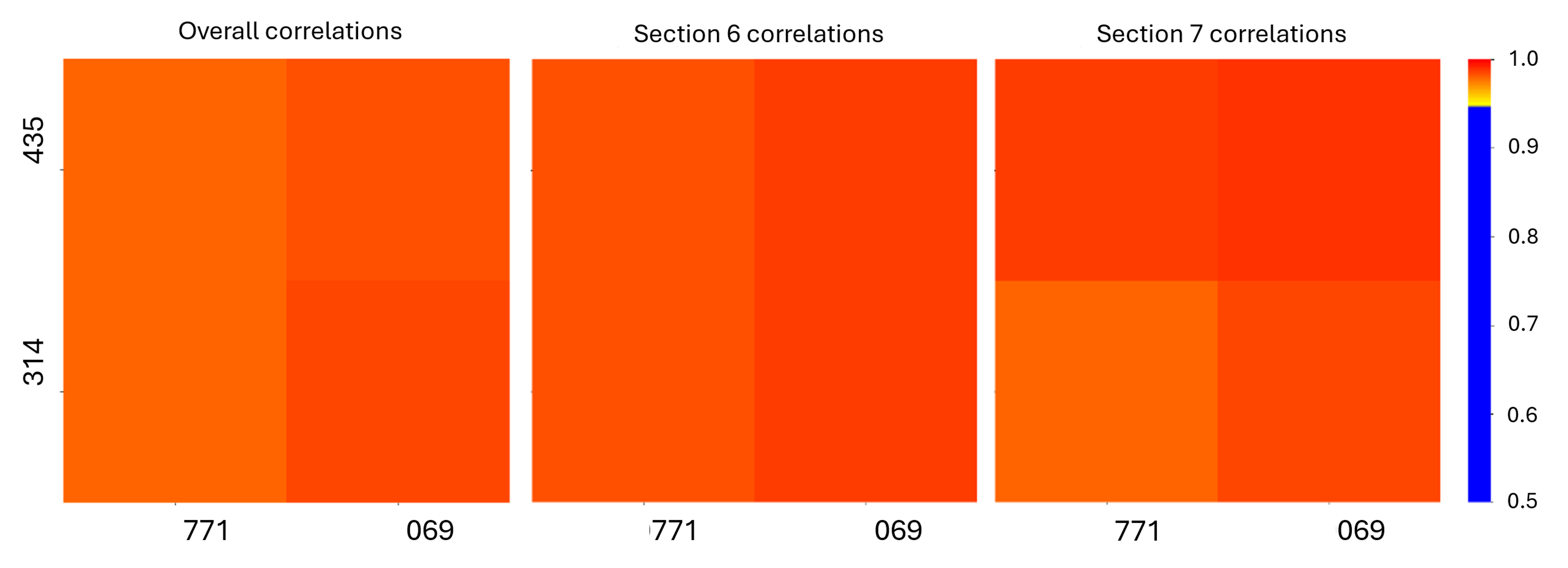}
        \caption{Correlation heatmap.}
        \label{fig:correlationheatmapall}
    \end{subfigure}
    \caption{Correlation similarity causing assignment ambiguity.}
    \label{fig:Correlation test case}
    \vspace{-4mm}
\end{figure}

\vspace{-3mm}
\subsection{Constraint-Guided Topology Recovery}

\begin{figure}[htbp]
    \centering
    \begin{subfigure}[b]{0.15\textwidth}
        \includegraphics[width=\textwidth]{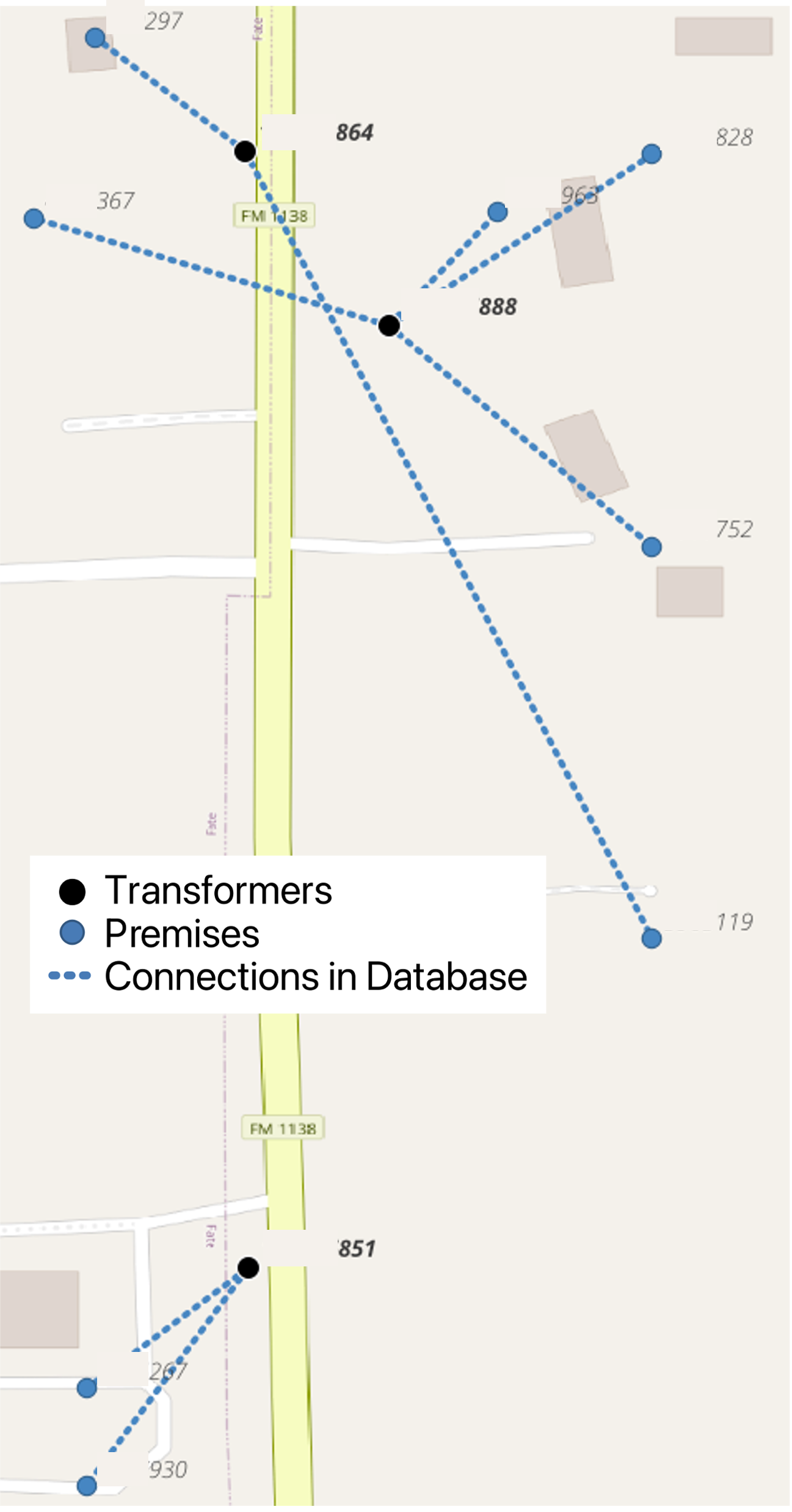}
        \caption{Database.}
    \end{subfigure}
    \hfill
    \begin{subfigure}[b]{0.15\textwidth}
        \includegraphics[width=\textwidth]{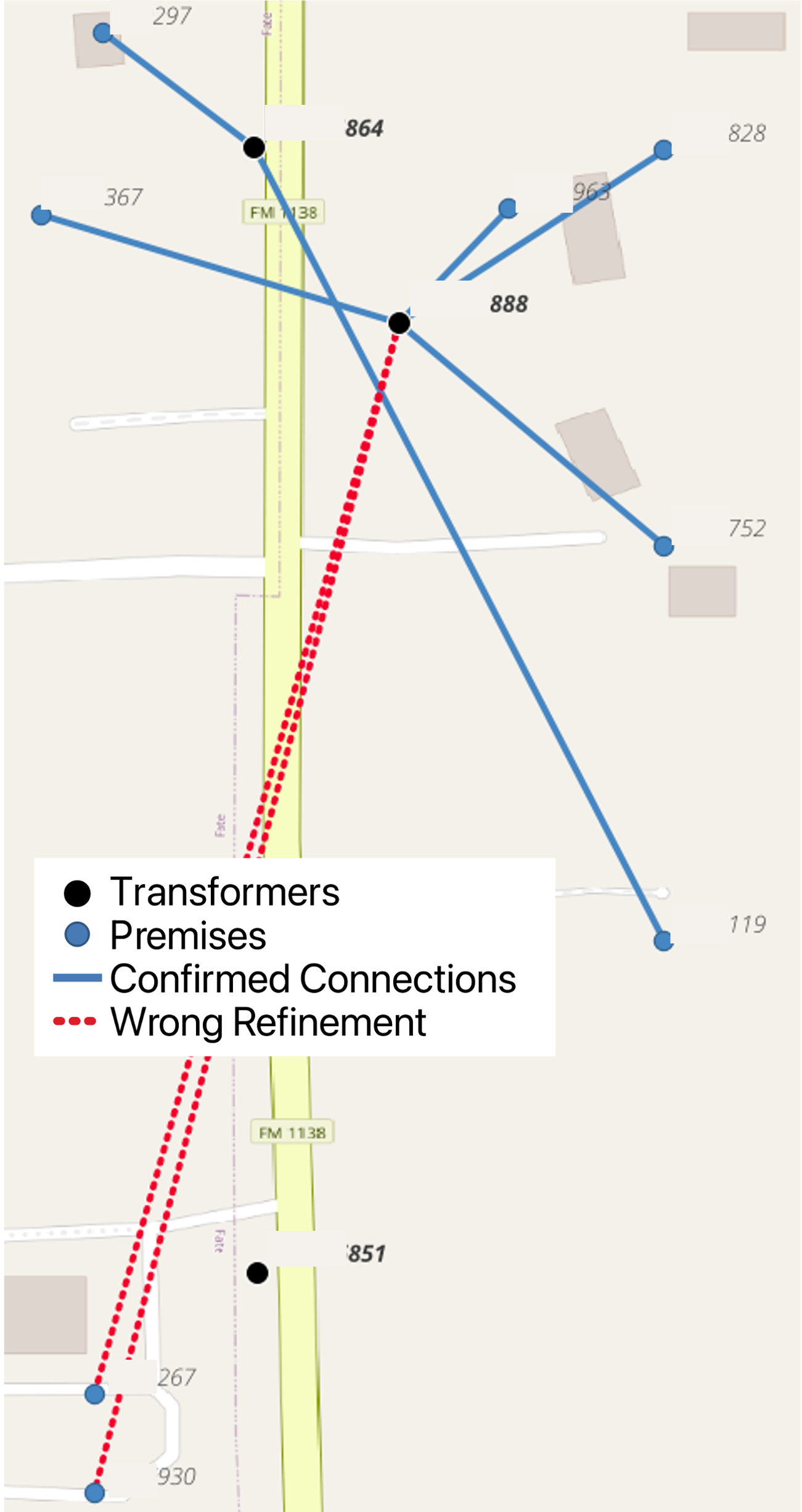}
        \caption{Correlation-only.}
    \end{subfigure}
    \hfill
    \begin{subfigure}[b]{0.15\textwidth}
        \includegraphics[width=\textwidth]{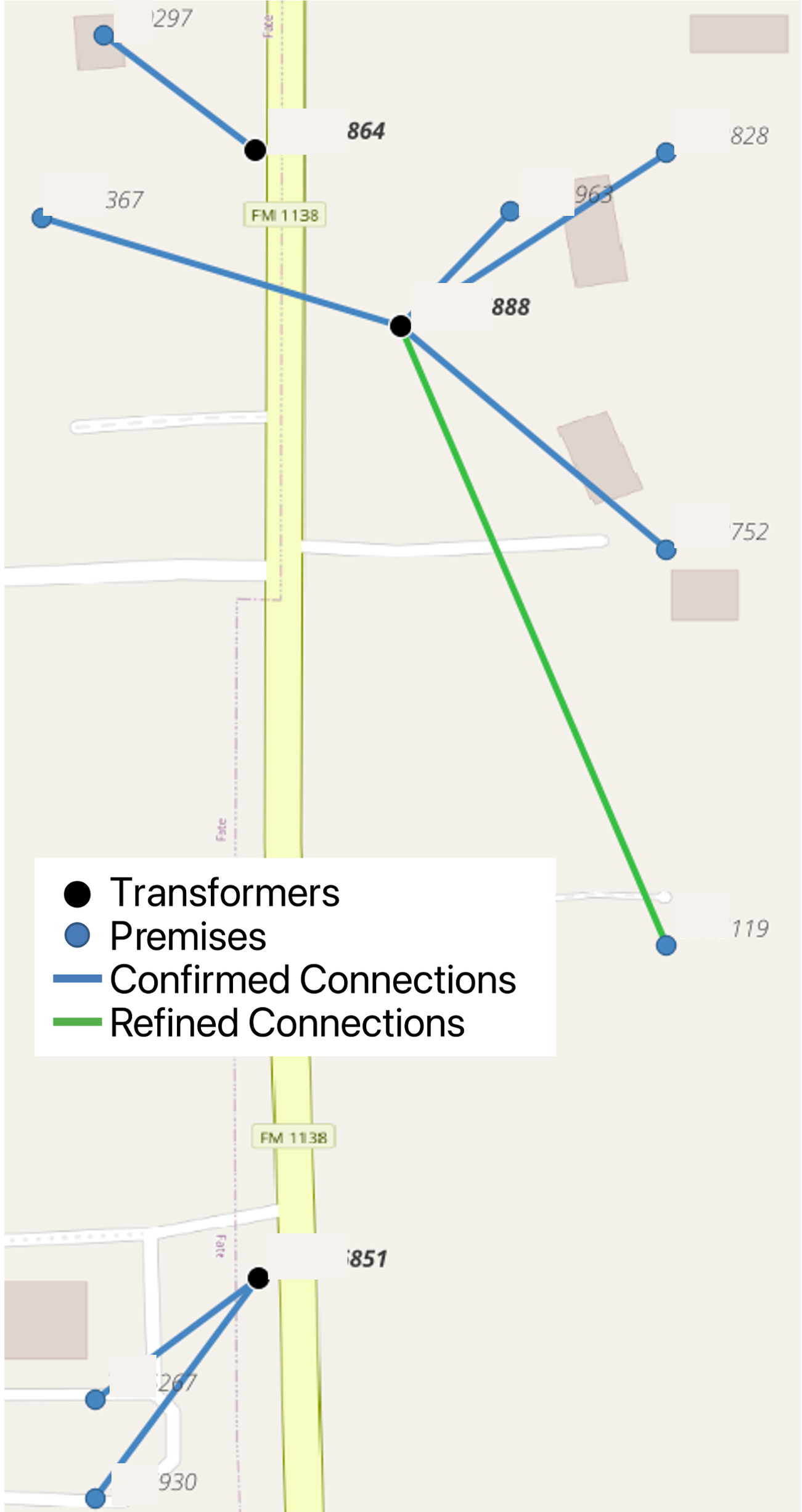}
        \caption{Proposed.}
    \end{subfigure}
    \caption{Topology correction comparison.}
    \label{fig:QGIS_of_Case_b}
    \vspace{-3mm}
\end{figure}
Figure~\ref{fig:QGIS_of_Case_b} compares database topology, correlation-only refinement, and the proposed constrained inference results. Electrical-only methods incorrectly merge neighboring transformer service regions, whereas constrained inference restores physically consistent assignments. By combining spatial feasibility, electrical consistency, and constraint-guided reassignment, the proposed solver corrects localized topology errors while preserving already correct connectivity, validating the localized refinement strategy developed in Section~V.

\vspace{-3mm}
\subsection{Reliability Evaluation for Deployment}

To support field deployment, inferred connections are assigned confidence levels measuring how strongly chosen assignments outperform alternative feasible mappings. Figure~\ref{fig:confidence_show} illustrates confidence improvements after constrained reassignment. Corrected assignments exhibit clear confidence increases, while already correct mappings remain stable, indicating the solver strengthens reliable assignments rather than introducing instability. Across feeders, confidence levels consistently improve, enabling utilities to prioritize verification only where ambiguity remains.

\begin{figure}[H]
\centering
\includegraphics[width=\columnwidth]{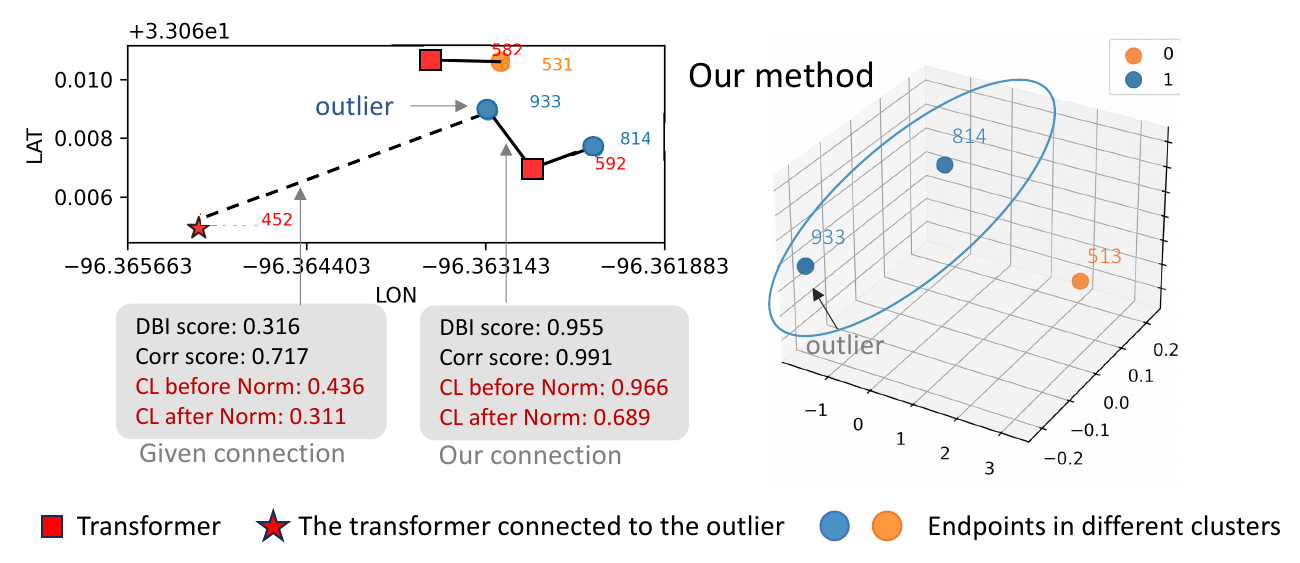}
\caption{Confidence improvement after constrained inference.}
\label{fig:confidence_show}
\vspace{-3mm}
\end{figure}

\vspace{-3mm}
\subsection{Physical Constraint Validation}

Assignments must also satisfy operational feasibility constraints. Table~\ref{tab:capacity_violation} shows transformers whose aggregated loads exceed rated capacities under incorrect connectivity. Constraint-guided reassignment resolves overload conditions, as illustrated in Figure~\ref{fig:overload_refinement}, confirming that refined topology estimates remain physically feasible.

\begin{table}[h!]
\centering
\caption{Example Transformer Capacity Violations}
\label{tab:capacity_violation}
\footnotesize
\begin{tabular}{c|c|c|c|c}
\toprule
ID & Rating & Peak & Limit & Status \\
\midrule
xxxxxx500 & 10 & 16.6 & 8 & Violation \\
xxxxxx802 & 10 & 11.3 & 8 & Violation \\
xxxxxx400 & 10 & 11.3 & 8 & Violation \\
xxxxxx041 & 10 & 9.4 & 8 & Violation \\
xxxxxx970 & 10 & 8.8 & 8 & Violation \\
xxxxxx683 & 15 & 13.0 & 12 & Violation \\
xxxxxx376 & 10 & 9.4 & 8 & Violation \\
xxxxxx718 & 15 & 12.8 & 12 & Violation \\
\bottomrule
\end{tabular}
\end{table}


\begin{figure}[t!]
    \centering
    \includegraphics[width=0.75\columnwidth]{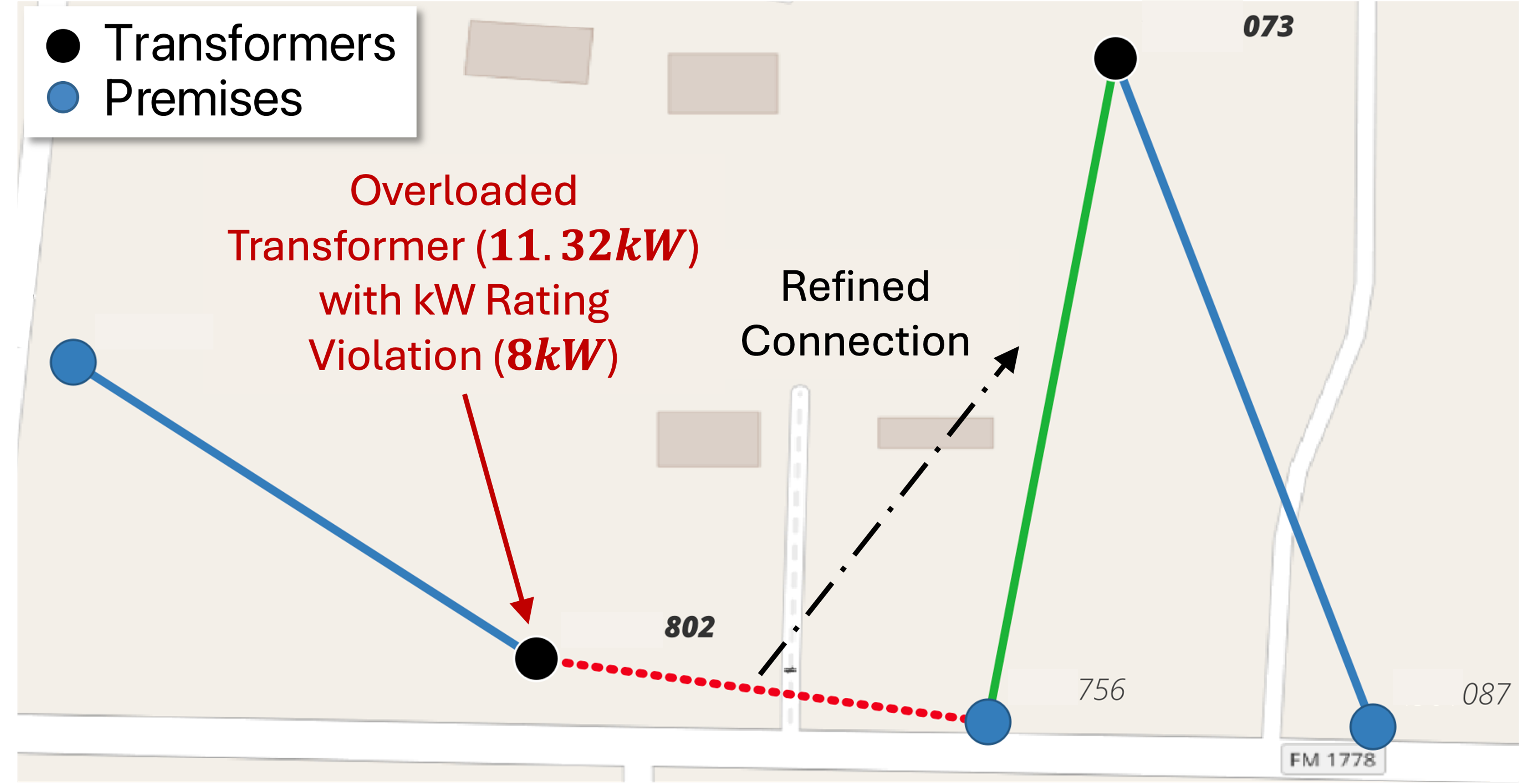}
    \caption{Constraint-guided correction resolving overload.}
    \label{fig:overload_refinement}
\end{figure}

\vspace{-3mm}
\subsection{Computational Scalability and Deployment Feasibility}

A key requirement for utility deployment is that topology refinement can be executed periodically across thousands of feeders without excessive computational cost. Consistent with the complexity analysis in Section~V, scalability is achieved by restricting inference to localized candidate transformer neighborhoods rather than recomputing feeder-wide connectivity. Electrical similarity and clustering operations are performed only within localized candidate sets, making practical runtime proportional to the number of detected inconsistencies rather than total feeder size.


\section{Conclusion and Future Work}
\label{sec:conclu}
This paper reformulates distribution topology identification as a constrained inference problem under heterogeneous and imperfect utility data. Instead of relying solely on electrical correlation or assuming perfect records, the proposed framework jointly enforces electrical consistency, spatial feasibility, and operational constraints to recover physically consistent connectivity while estimating the reliability of each inferred connection. Validation on large-scale operational feeders demonstrates that correlation-based approaches alone can fail in dense networks due to indistinguishable voltage behavior across neighboring transformers. By combining multi-source information with constraint-guided inference, the proposed method reliably corrects localized connectivity errors while preserving system-wide observability and maintaining computational scalability. The resulting reliability metric further supports practical deployment by enabling utilities to prioritize verification efforts. Overall, the results show that trustworthy topology models can be recovered even when data quality is uneven, provided inference explicitly incorporates physical and operational constraints. Future work will investigate extensions toward dynamic topology tracking and real-time validation in distribution system operations.

\appendices

\vspace{-3mm}
\section{Extended Confidence Evaluation Across Feeders}
\label{app:confidence_extended}

This appendix provides additional confidence evaluation results across multiple feeders. Figure~\ref{fig:test case} shows four representative feeders with distinct spatial configurations and loading patterns. For each case, confidence levels are evaluated before and after constrained inference.  Nodes with low confidence typically correspond to locations where neighboring transformers exhibit highly similar electrical behavior or where spatial records contain ambiguity. Across all cases, constrained inference consistently improves confidence scores for corrected assignments while preserving stable scores for already consistent mappings. This indicates that the proposed framework improves assignment reliability rather than merely altering connectivity. From an operational perspective, utilities can use confidence levels to prioritize manual verification only in ambiguous regions, reducing field inspection costs while preserving system-wide observability.

\begin{figure*}[t!]
    \centering
    \begin{subfigure}[b]{0.48\textwidth}
        \includegraphics[width=\textwidth]{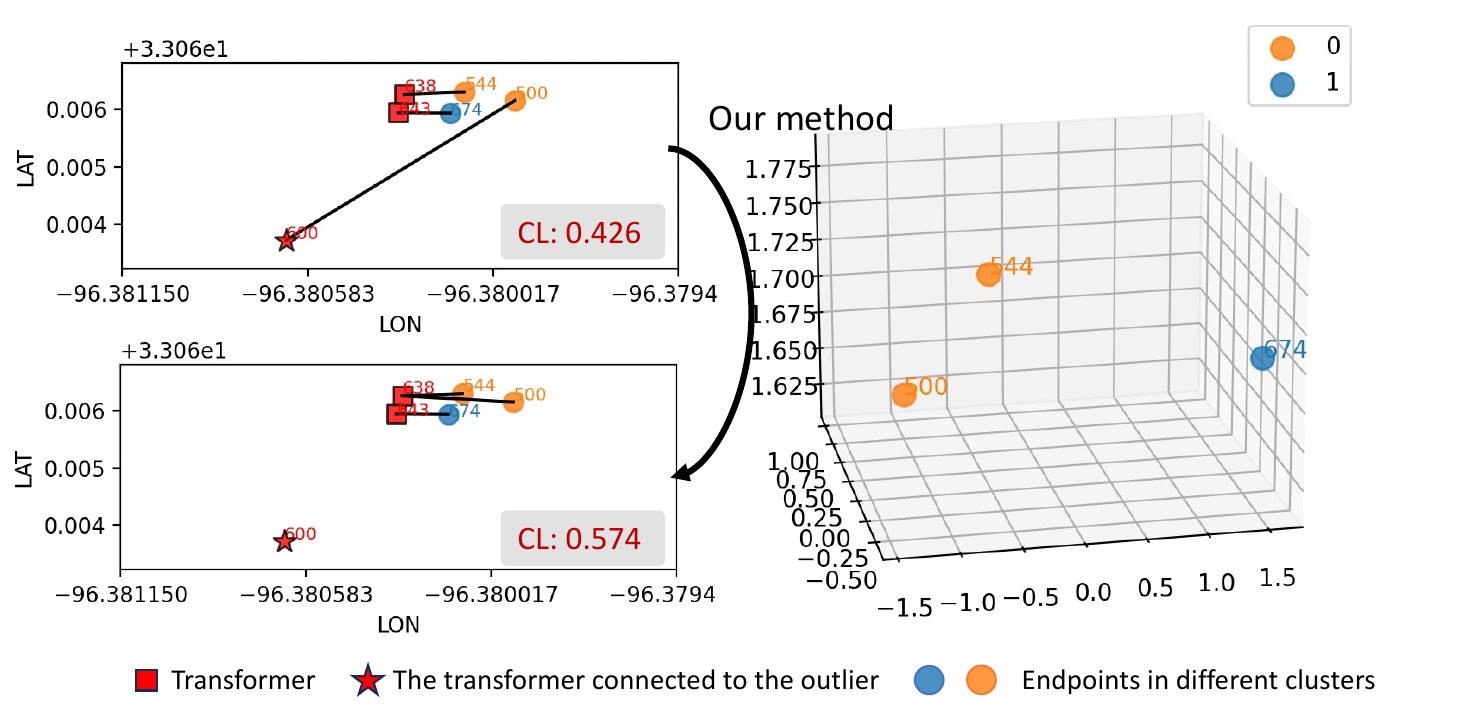}
        \caption{Case 1}
    \end{subfigure}
    \hfill
    \begin{subfigure}[b]{0.48\textwidth}
        \includegraphics[width=\textwidth]{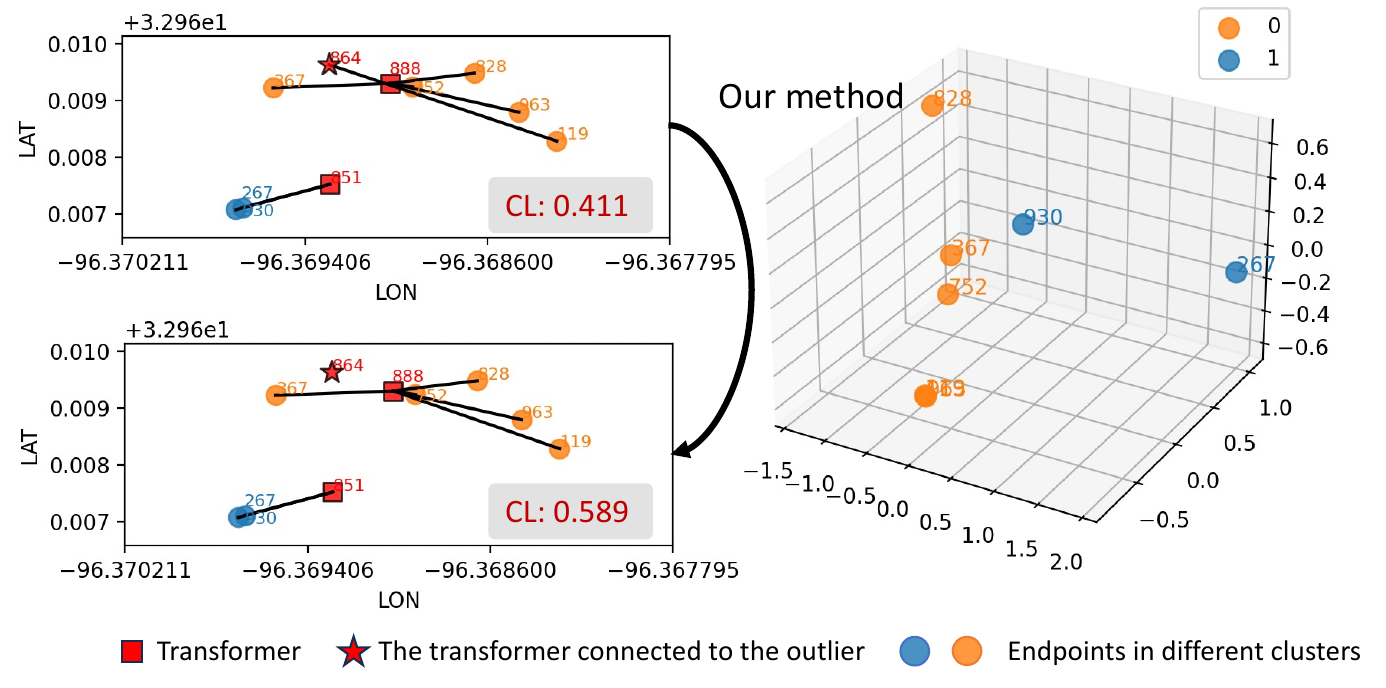}
        \caption{Case 2}
    \end{subfigure}
    \\[1em]
    \begin{subfigure}[b]{0.48\textwidth}
        \includegraphics[width=\textwidth]{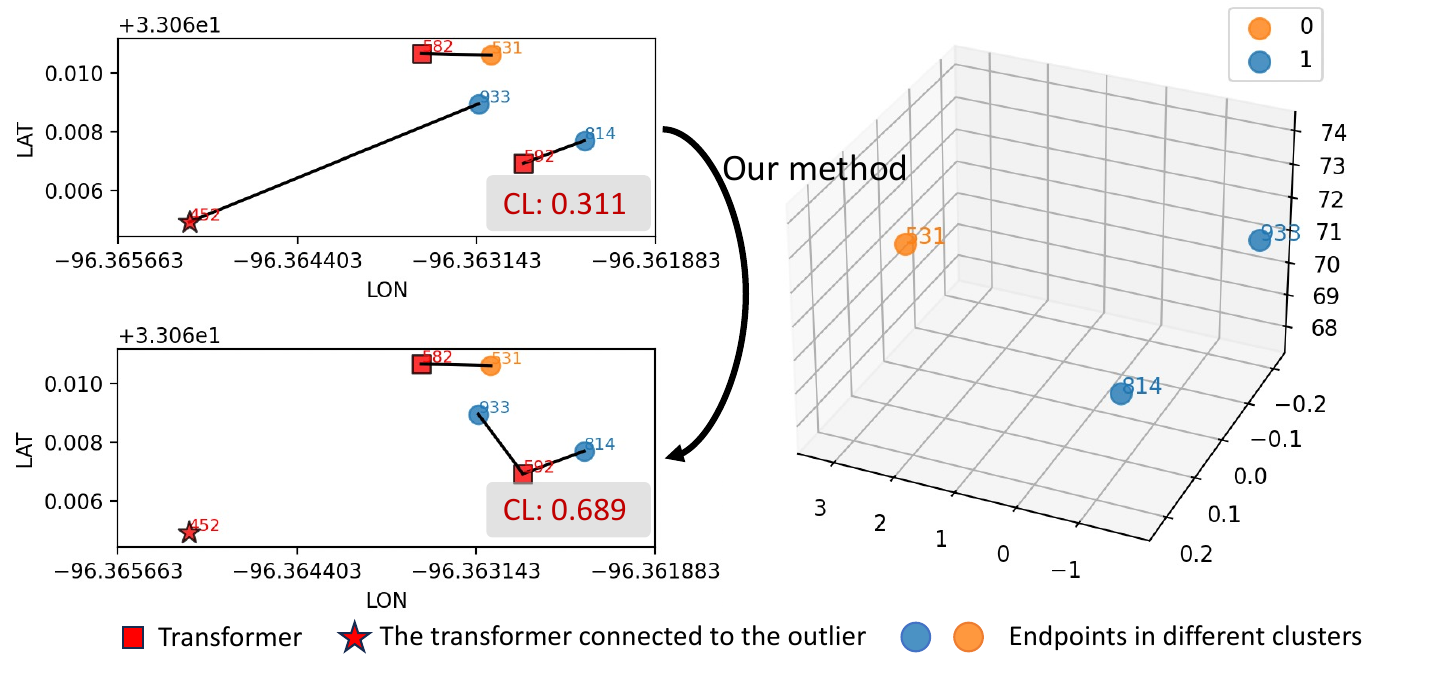}
        \caption{Case 3}
    \end{subfigure}
    \hfill
    \begin{subfigure}[b]{0.48\textwidth}
        \includegraphics[width=\textwidth]{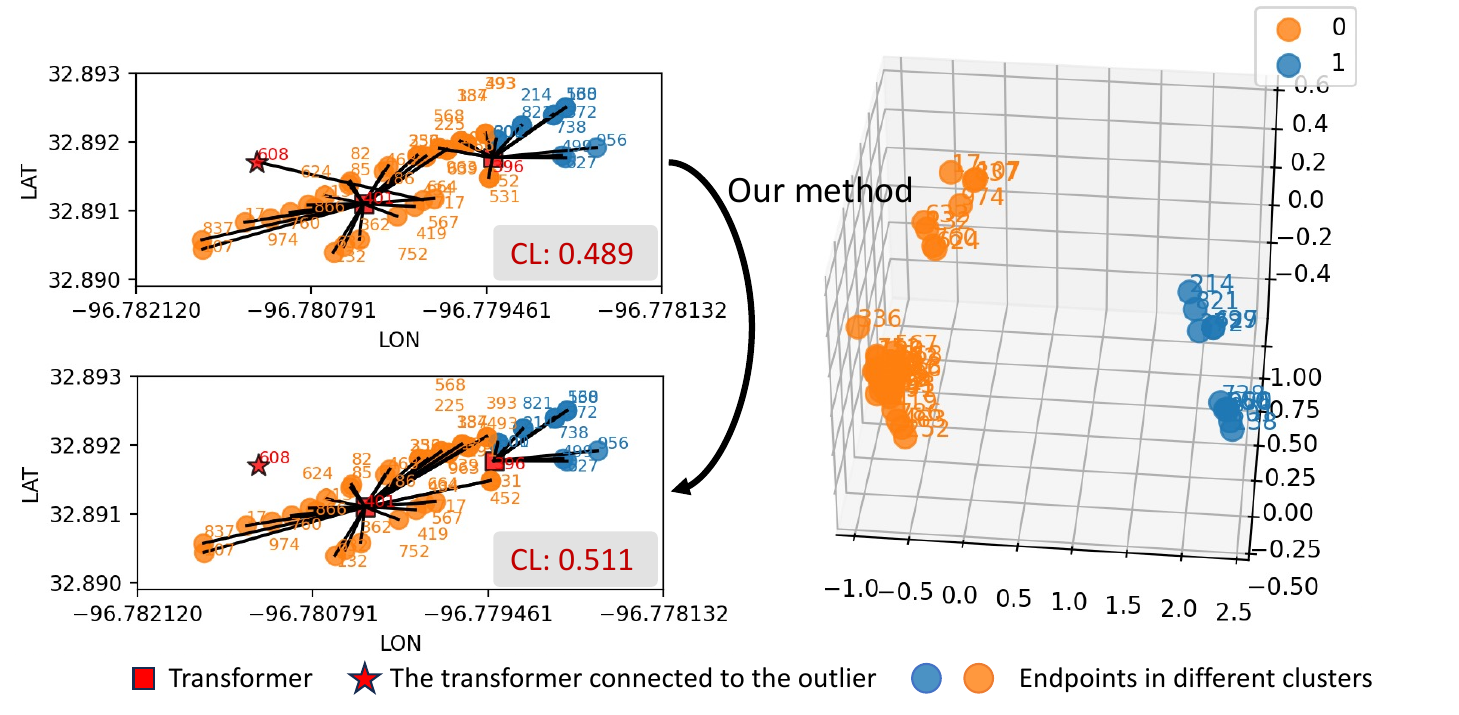}
        \caption{Case 4}
    \end{subfigure}
    \caption{Confidence evaluation across feeders with diverse configurations.}
    \label{fig:test case}
\end{figure*}

\vspace{-3mm}
\section{Sensitivity and Robustness Analysis}
\label{app:sensitivity}

A practical concern in localized topology refinement is whether inference results depend strongly on parameter choices controlling candidate transformer search regions. 
To evaluate robustness, we vary the number of candidate transformers considered during local reassignment and observe resulting assignments across four representative regions. Table~\ref{tab:outlier_assignments} summarizes assignments obtained under different candidate neighborhood sizes. Results remain unchanged for small and moderate candidate sets, indicating stable inference behavior. Only when candidate sets become unnecessarily large do minor assignment variations appear. Figure~\ref{fig:sensitivity analysis} further visualizes reassignment outcomes across these regions. The stability of assignments confirms that constrained inference relies primarily on local physical consistency rather than sensitive parameter tuning. 

\begin{table}[ht]
\centering
\caption{Transformer assignment results across regions under different candidate sizes.}
\begin{tabular}{c|c|c|c|c}
\toprule
 & Region 1 & Region 2 & Region 3 & Region 4 \\
\cline{1-5}
Outlier & xxxx500 & xxxx119 & xxxx933 & xxxx567 \\
\hline
TA (K=2) & xxxxxx638 & xxxxxx888 & xxxxxx592 & xxxxx401 \\
TA (K=3) & xxxxxx638 & xxxxxx888 & xxxxxx592 & xxxxx401 \\
TA (K=4) & xxxxxx632 & xxxxxx864 & xxxxxx592 & xxxxx520 \\
\bottomrule
\end{tabular}
\label{tab:outlier_assignments}
\end{table}

\begin{figure*}[h]
    \centering
    \begin{subfigure}[b]{0.45\textwidth}
        \includegraphics[width=\textwidth]{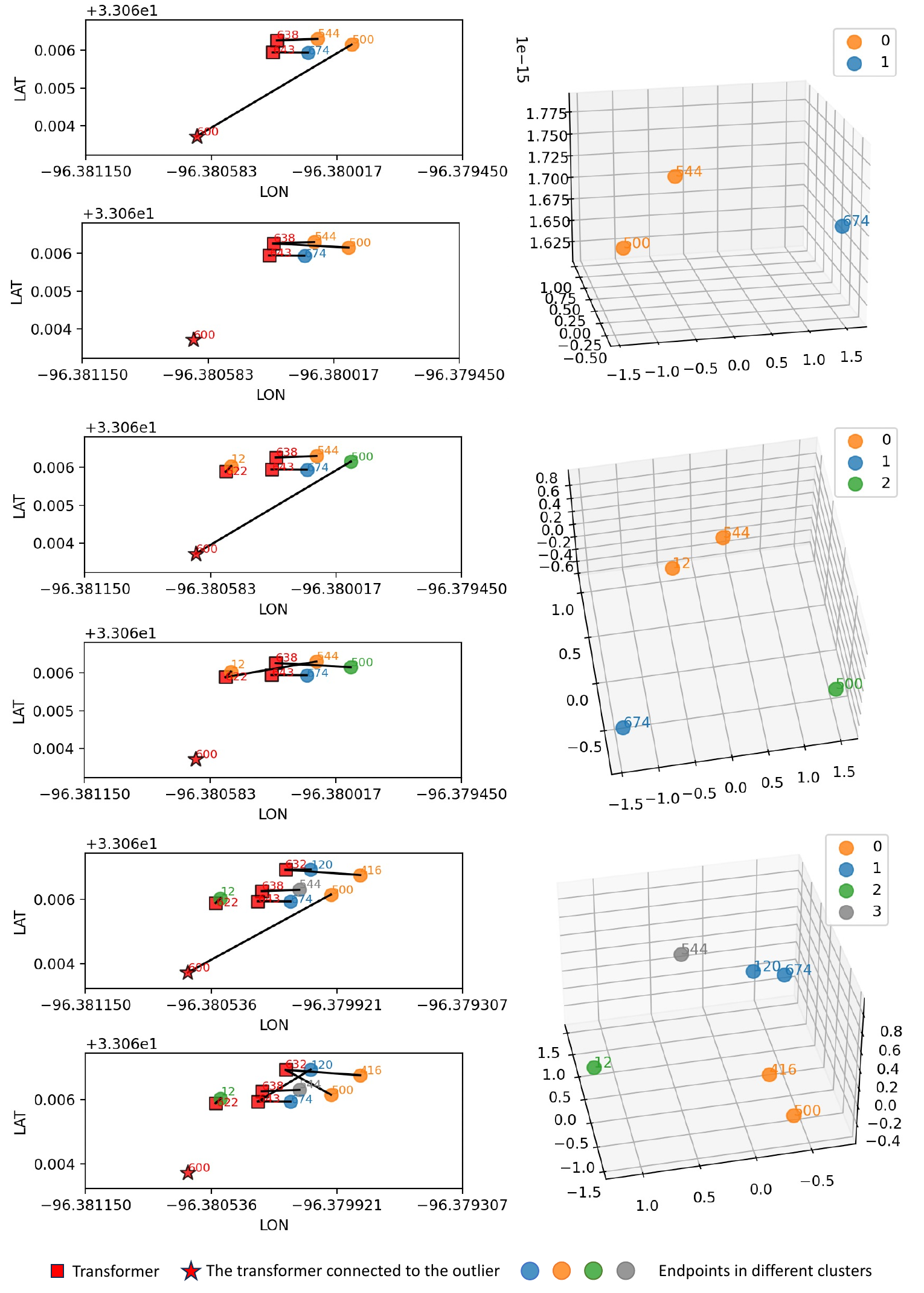}
        \caption{Region 1}
    \end{subfigure}
    \hfill
    \begin{subfigure}[b]{0.45\textwidth}
        \includegraphics[width=\textwidth]{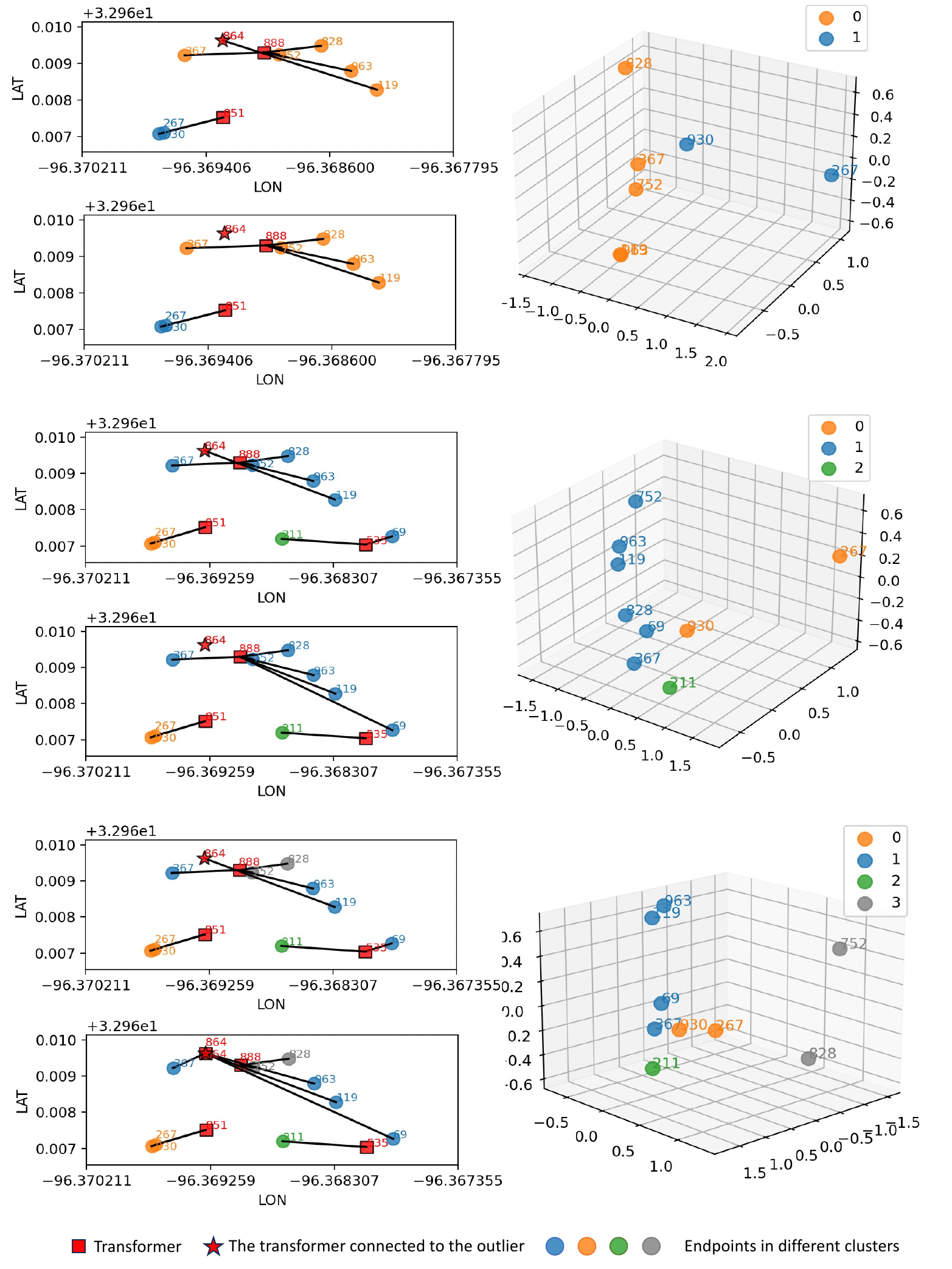}
        \caption{Region 2}
    \end{subfigure}
    \\[-0em]
    \begin{subfigure}[b]{0.45\textwidth}
        \includegraphics[width=\textwidth]{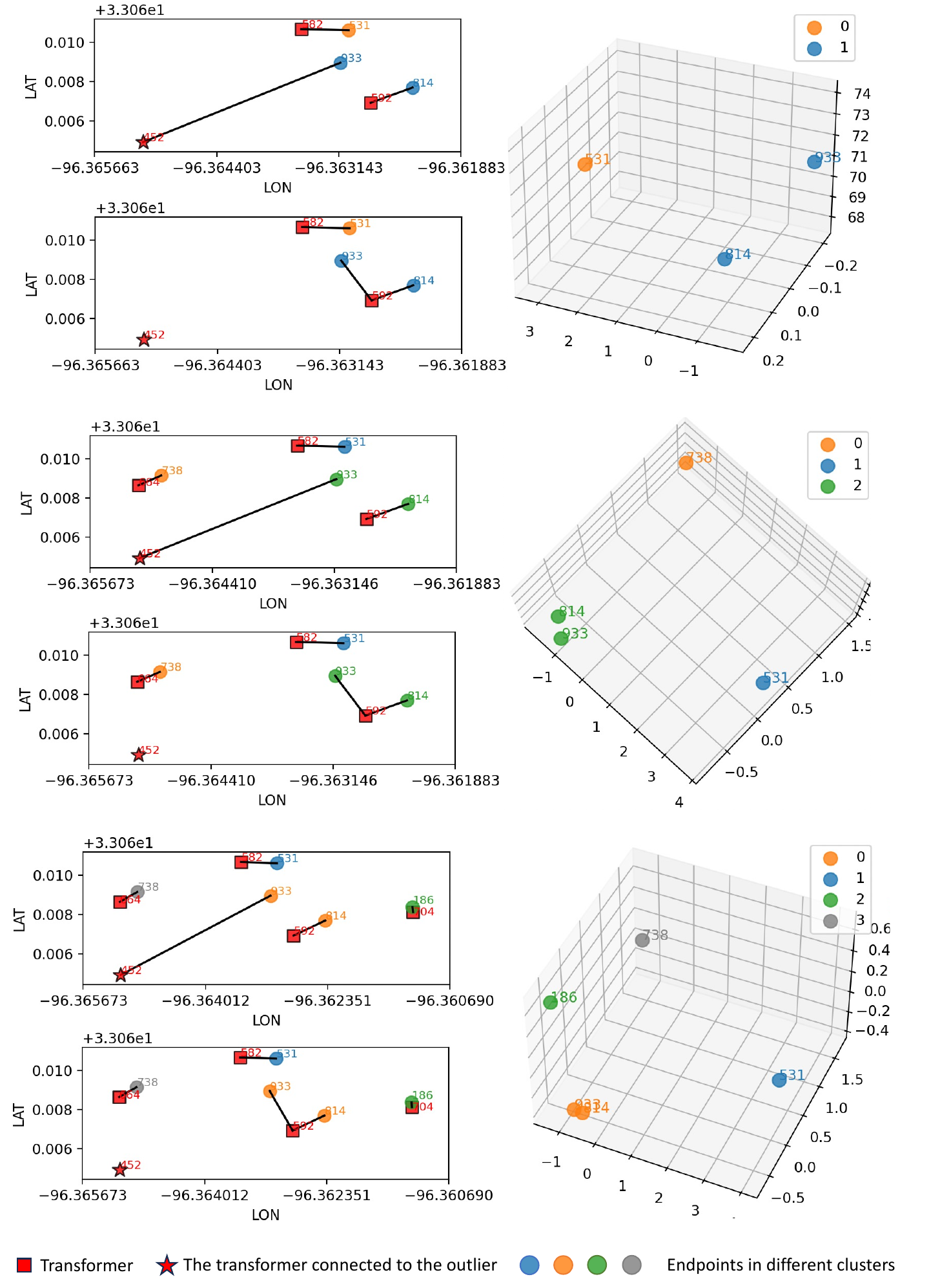}
        \caption{Region 3}
    \end{subfigure}
    \hfill
    \begin{subfigure}[b]{0.45\textwidth}
        \includegraphics[width=\textwidth]{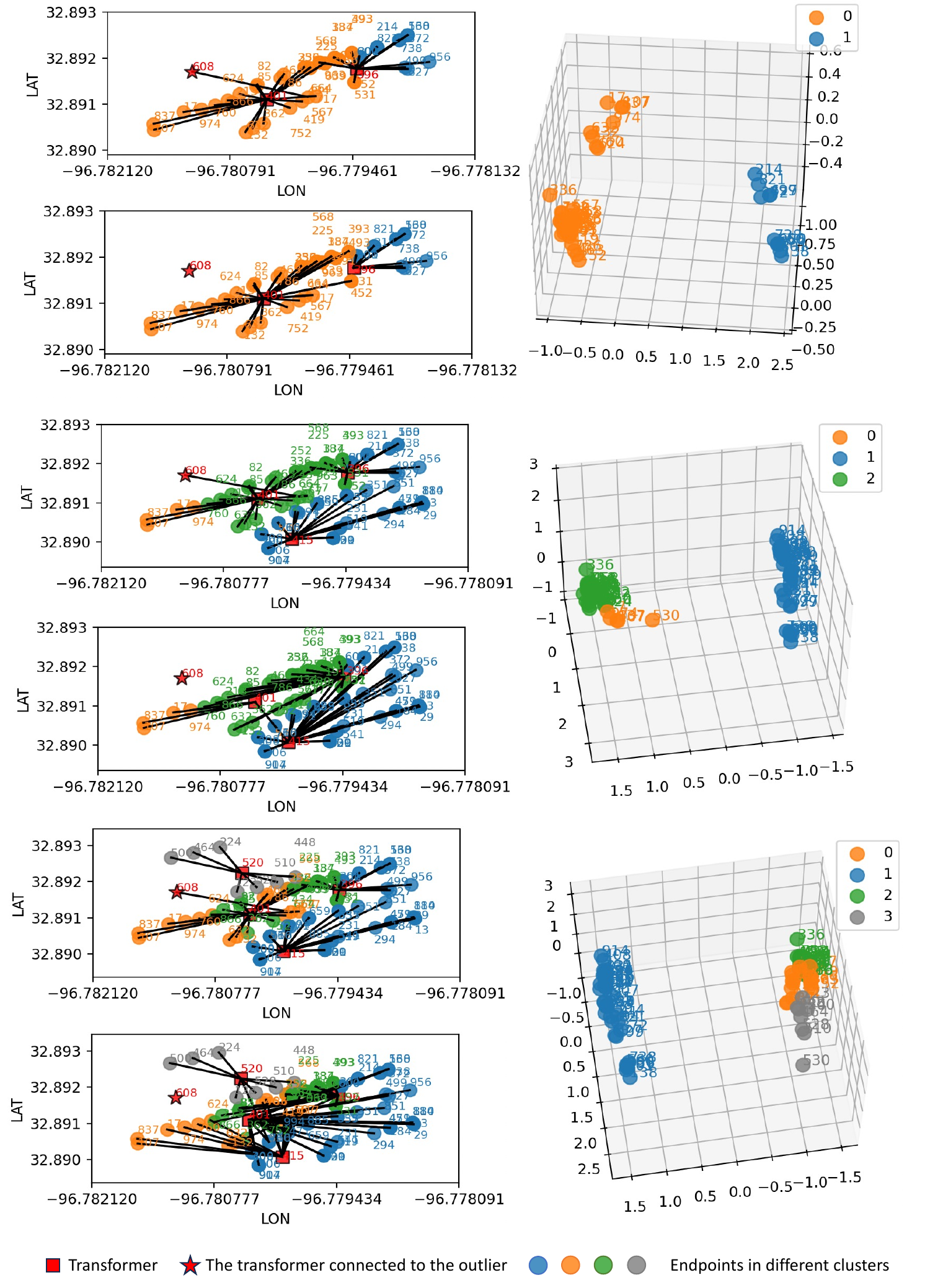}
        \caption{Region 4}
    \end{subfigure}
    \caption{Sensitivity analysis across regions.}
    \label{fig:sensitivity analysis}
\end{figure*}

\bibliographystyle{IEEEtran}
\bibliography{IEEEabrv,reference}

\end{document}